\documentclass{article}

\usepackage{iclr2027_conference,times}
\iclrfinalcopy
\usepackage[dvipsnames]{xcolor}
\usepackage{colortbl}

\usepackage{hyperref}
\usepackage{url}
\usepackage{booktabs}
\usepackage{amsfonts}
\usepackage{amsmath}
\usepackage{amsthm}

\newtheorem{proposition}{Proposition}
\newtheorem{lemma}{Lemma}

\newtheorem{theorem}{Theorem}
\usepackage{amssymb}
\usepackage{nicefrac}
\usepackage{microtype}
\usepackage{graphicx}
\usepackage{subcaption}
\usepackage{algorithm}
\usepackage{algpseudocode}
\usepackage{enumitem}
\usepackage{natbib}
\usepackage{cleveref}
\usepackage{multirow}
\usepackage{makecell}
\usepackage{stmaryrd}
\SetSymbolFont{stmry}{bold}{U}{stmry}{m}{n}

\usepackage{amsmath,amsfonts,bm}

\def\eqref#1{equation~\ref{#1}}

\def\1{\bm{1}}
\newcommand{\train}{\mathcal{D}}

\def\eps{{\epsilon}}

\def\rd{{\textnormal{d}}}

\def\rz{{\textnormal{z}}}

\def\vmu{{\bm{\mu}}}
\def\vtheta{{\bm{\theta}}}

\def\vb{{\bm{b}}}

\def\vh{{\bm{h}}}

\def\vu{{\bm{u}}}

\def\vw{{\bm{w}}}
\def\vx{{\bm{x}}}

\def\vz{{\bm{z}}}

\def\mI{{\bm{I}}}

\def\mW{{\bm{W}}}

\def\mSigma{{\bm{\Sigma}}}

\DeclareMathAlphabet{\mathsfit}{\encodingdefault}{\sfdefault}{m}{sl}
\SetMathAlphabet{\mathsfit}{bold}{\encodingdefault}{\sfdefault}{bx}{n}

\def\gE{{\mathcal{E}}}

\def\gH{{\mathcal{H}}}

\def\gN{{\mathcal{N}}}
\def\gO{{\mathcal{O}}}

\def\sP{{\mathbb{P}}}

\def\sX{{\mathbb{X}}}
\def\sY{{\mathbb{Y}}}

\newcommand{\E}{\mathbb{E}}

\newcommand{\R}{\mathbb{R}}

\newcommand{\softmax}{\mathrm{softmax}}

\newcommand{\KL}{D_{\mathrm{KL}}}

\DeclareMathOperator*{\argmin}{arg\,min}

\DeclareMathOperator{\Tr}{Tr}

\makeatletter
\renewcommand{\eqref}[1]{\textup{\tagform@{\ref{#1}}}}
\makeatother

\def\mOmega{{\bm{\Omega}}}

\def\Ideal{p}
\def\IdealR{\Ideal_{r}}
\def\IdealF{\Ideal_{f}}

\def\Proxy{\hat{p}}
\def\ProxyR{\Proxy_{r}}
\def\ProxyF{\Proxy_{f}}
\def\ProxyShift{(\log{\ProxyR} - \log{\Proxy})}

\def\LogitI{{\bm{\ell}_{\vtheta}}}

\def\LogitU{{\log{\ProbitU}}}
\def\LogitD{{\bm{\ell}_{\vtheta_{u}}}}

\def\ProbitI{{f_{\vtheta}}}
\def\ProbitR{{f_{\vtheta_{r}}}}
\def\ProbitU{{q}}
\def\ProbitD{{f_{\vtheta_{u}}}}

\def\gammaU{\zeta}
\def\Deltagamma{\mathrm{Z}}

\def\n{|\train|}
\def\nr{|\train_r|}
\def\nf{|\train_f|}

\def\TRAIN{\mathrm{t}}
\def\TEST{\mathrm{t}}
\def\FORGET{\mathrm{f}}

\DeclareMathOperator{\lse}{LSE}
\DeclareMathOperator{\bernoulli}{\mathcal{B}}

\newcommand{\kl}[3][]{\KL^{#1}\!\left( #2 \,\middle\Vert\, #3 \right)}

\newcommand{\sach}[2]{\left( #1 \mid #2 \right)}
\newcommand{\sachq}[3]{_{#1}\!\left( #2 \mid #3 \right)}

\def\sq{\mathrm{Q}}

\def\lda{\mathrm{LDA}}
\def\qda{\mathrm{QDA}}
\def\Acc{\mathrm{Acc}}

\def\RTE{\mathrm{RTE}}

\def\class{\textit{Class}}
\def\subclass{\textit{Subclass}}

\def\LDAI{\texttt{LDA}}
\def\QDAI{\texttt{QDA}}
\def\LDAII{\texttt{LDA-Mix}}
\def\QDAII{\texttt{QDA-Mix}}
\def\LDAIJ{\texttt{LDA-2C}}
\def\DiracI{\texttt{DIR}}
\def\DiracIJ{\texttt{DIR-2C}}

\def\ie{\textit{i.e.}}

\title{Behavioral Guarantees for Proxy Based Unlearning}

\author{Virgile Dine \\
  Centre Inria de l'Universit\'e de Rennes \\
  France \\
  \texttt{virgile.dine@inria.fr} \\
  \And
  Teddy Furon \\
  Centre Inria de l'Universit\'e de Rennes \\
  France \\
  \texttt{teddy.furon@inria.fr} \\
}

\begin{document}

\maketitle

\begin{abstract}
This paper generalizes recent proxy-based unlearning methods and proves theoretical guarantees about the behavior of the resulting unlearned model:
upper bounds on the averaged Kullback-Leibler divergence from the ideal retrained.
We model approximate unlearning as a constrained optimization problem
and interpret a family of solutions as introducing a scaled unlearning signal in the output space.
The unlearning signal arises from proxies of the posterior data distributions.
Under a design condition that the practitioner can empirically estimate, its scale is adapted to the proxies to ensure the behavioral upper bounds.
This framework relies on the structure of the data distributions in order to create proxies.
If need be, the target serves as a teacher to distill the update in the weights.
Our approach is experimentally validated over two
forgetting scenarios as reaching the closest classifier to the ideal retrained model.
\end{abstract}

\section{Introduction}
\label{sec:intro}

Approximate machine unlearning removes the influence of specific training data from an already trained classifier.
This avoids retraining it from scratch.
Applications range from privacy (the right to be forgotten)
to safety/security by removing corrupted/poisoned training data. Our conditional guarantees are behavioral at output level. We bound the Kullback-Leibler divergence to the ideal model in expectation over inputs. They are averaged and not per-example privacy guarantees.

Classical gradient-based methods construct asymmetrical losses that treat retain or forget samples differently.
They are iterative, hyperparameter-sensitive, and lack a stopping criterion.
This leads to unstable behaviors.
Their efficacy is often gauged by misleading accuracy metric interpretations or oracles, without accounting for the ideal retrained model.

This paper first defines the unlearning signal as the logits shift between the ideal and the initial output distributions.
Our method estimates this shift in the output space using posteriors of a pair of proxies.
Section~\ref{sec:our_method} bounds the gap to the ideal retrained model by an initial divergence plus a modeling shift.
It proves that the bound is small whenever the two proxies' errors compensate for each other.
This holds under two conditions: i) the scale of the unlearning signal remains under the maximal $\eta_{\max}$ estimated by Alg.~\ref{alg:etamax}, ii) the pair of proxies is well designed, a condition whose retain part the practitioner can empirically test before proceeding with the unlearning.
This unlearning signal is added on top of a black-box model, or cast within the neural network parameters via distillation.
With the teacher fixed, this final step is fast, stable, and needs no stopping criterion.

A key element of our method is the choice of the proxies.
We construct them by explicitly modeling the data distributions rather than by learning posteriors.
Section~\ref{sec:lda2c} proposes a family of proxies ranging from naive marginalization
to more advanced modeling specifically designed for unlearning.

Our contributions are the following.
\emph{(i)} Theoretical guarantees on the behavior of proxy-based unlearned model (Sect.~\ref{sec:our_method}), which apply to previous and future output-correction methods (see App.~\ref{app:unify}) with an explicit condition on the design of the unlearning signal.
\emph{(ii)} A safe adaptive scaling $\eta_{\max}$ (Alg.~\ref{alg:etamax}) derived by a simple $1$-d line search and its practical implementation with interpretable unlearning signals stemming from efficient data modeling proxies (Sect.~\ref{sec:lda2c}).
\emph{(iii)} An experimental validation across four architectures, including a foundation model backbone, and two scenarios of unlearning
demonstrating the versatility of the proposed method (Sect.~\ref{sec:experiments}).

\section{Related work}
\label{sec:related}

\textbf{Closed-form Modeling.} Class-conditional densities have a rich history in statistics.
Gaussian discriminant analysis (LDA, QDA) provides the canonical example, with regularized variants \citep{Friedman1989rda} and Bayesian counterparts \citep{Murphy2012ml,Gelman2013bda}.
More flexible options include per-class Gaussian mixtures \citep{Dempster1977em,McLachlan2000mixtures} and constrained covariance structures \citep{Ghahramani1996mfa}.
Our method  leverages these fast closed-form approaches.
To our knowledge, no prior work has exploited the closed-form modeling by lifting the label space in a higher dimension to model the membership as a subclass, as proposed in Sect.~\ref{sec:lda2c}.

\textbf{Machine unlearning.} It is the art of removing the influence of some training data without learning a new model from scratch (see App.~\ref{app:machine_unlearning} for a deeper overview).  The gold standard is exact removal by design. Exact techniques rely on specific architectures or data structures to avoid full retraining~\citep{cao2015machine}. For instance, ~\citet{Bourtoule2021SISA} propose the method SISA, which splits the training data into shards and stores intermediate checkpoints for each shard. The trade-off between memory footprint and utility is the main limitation at scale, especially when the data to be forgotten spans multiple shards.

\textbf{Approximate unlearning.} It aims at making the unlearned model statistically indistinguishable from the ideal model that never saw the forget data.
The dominant paradigm, so-called first-order methods, consists of gradient-based updates of the NN parameters.
Gradient ascent on the forget-set loss is the simplest primitive, many methods instead change the labels~\citep{graves2021amnesiac,chen2023boundary}.
More surgical method like SalUn~\citep{Fan2024salun} restricts updates to the weights most responsible for forgotten data.
Distillation variants enforce retention jointly with forgetting.
For instance, SCRUB trains a student by minimizing KL divergence with the original model on the data to retain while maximizing it on the data to forget~\citep{Kurmanji2023scrub}.
\citet{Chundawat2023badteacher} use an incompetent teacher guiding the student to suppress class knowledge.
\citet{Tarun2023fast} propose an impair-repair scheme: an error-maximizing noise degrades class knowledge, then a repair step restores retain-set accuracy.
These first-order methods are unified into either a multi-objective or a constrained optimization~\citep{zhouefficient, dine:hal-05247290}: unlearning the forget data while preserving the model utility.
These papers outline the conflicting gradients of the objectives and propose ad-hoc combinations.
Yet, the iterative updates produce a trajectory with no clear stopping criterion.
Moreover, \citet{mavrothalassitis2026ascent} showed that dependence between forget and retain data prevents these first-order methods from unlearning properly.

Our formulation of the unlearning~\eqref{eq:main-underlying} is also framed as a constrained optimization problem.
Yet, it has a closed-form solution which offers a theoretically grounded model correction, so-called unlearning signal.
Distillation into the model is only a final option where the unlearning signal acts as a fixed teacher, so no catastrophic stopping criterion is needed.

From the audit point of view, the KL divergence from the ideal model to the empirically unlearned NN has been adopted as a natural criterion~\citep{golatkar2020eternal,Chien2024certified,georgiev2025attribute,ribero2026regularizedfdivergencekerneltests}.
Yet, nearly all theoretical results are stated in terms of NN properties rather than of data distribution.
Indeed, the NeurIPS unlearning competition \citet{Triantafillou2024challenge} relies on $(\epsilon,\delta)$ certification~\citep{Sekhari2021remember} or membership inference attacks (MIAs)~\citep{Shokri2017mia,Carlini2022lira,Hayes2024ulira}
 to evaluate unlearning.
 Those quantities are accessible to the auditor, but they cannot be used as guidance for the unlearner.

Parametric models are rare in the machine unlearning literature.
\citet{georgiev2025attribute} propose linear data model matching as a way to predict the output of the retrained model.
Its instantiation is fast but requires to learn in advance thousands of models on random splits of the full training set.
They enter our framework as their update can be written as a logit shift (see App.~~\ref{app:unify}).
On the contrary, section~\ref{sec:lda2c} details one possible build of fast pairs of proxies per class by lifting the label to $(y,s)$, where $y$ is a class and $s$ indicates membership of the forget set.

As for LLMs, recent approaches add a logits correction.
ULD~\citep{ji2024uld} trains a small assistant overfitting the forget set and subtracts its logits from the initial model.
UCD~\citep{suriyakumar2025ucd} uses contrastive approach with two auxiliary models.
T3-Unlearning~\citep{block2026t3} tilts and raises the temperature.
In all three, their scaling is a free hyper-parameter tuned by hand.
Our framework provides theoretical grounding (see App.~\ref{app:unify}), thus the propositions of Sect.~\ref{sec:our_method} do apply to these methods and let us derive the adaptive scaling $\eta_{\max}$ (Alg.~\ref{alg:etamax}) instead of fixed hyperparameter.

\section{Behavioral guarantees of proxy modeling}
\label{sec:our_method}

\textbf{Notations} (summarized in Table~\ref{tab:glossaire}, App~\ref{app:notations}).
The original training set $\train = \train_r \cup \train_f$ is composed by the forget set $\train_f$ to be erased and the retain set $\train_r$ to be preserved.
Given a neural network (NN : functions from the input space $\sX$ to the \emph{logit} space $\R^C$ with $C$ classes) $\LogitI$ trained on $\train$, machine unlearning seeks an efficient procedure producing NN $\LogitD$ as if it had been trained on $\train_r$ alone.
The softmax operator yields predicted class probabilities a.k.a. \emph{probits}
$\ProbitI(\vx) \doteq \softmax(\LogitI(\vx))$.
Conversely, $\LogitI(\vx) = \log(\ProbitI(\vx)) + \gamma(\vx)\1$ for some $\gamma(\vx)\in\R$.
We denote by $\Ideal$ the initial joint distribution of the data $(y,\vx)$.
The induced marginal and conditioned true distributions are denoted $\Ideal_\vx$, $\Ideal_y$ and $\Ideal\sach{y}{\vx}$, $\Ideal\sach{\vx}{y}$, or simply $\Ideal$ when the context is clear.
The Kullback-Leibler divergence from distribution $\Ideal\sach{y}{\vx}$ to $q\sach{y}{\vx}$ is denoted as
\[\kl[y\mid\vx]{q}{p}\doteq\sum_{y=1}^C q\sach{y}{\vx}\log(q\sach{y}{\vx}/p\sach{y}{\vx}).\]
Its expectation over $\vx\sim\Ideal_\vx$ is denoted $\E\kl[y\mid\vx]{q}{p}$ and is the main quantity of our study. Thus our bounds are stated on average over the data distribution, not per example. The subscript $r$ marks every variant attached to $\train_r$.

\subsection{Formalism of the problem}
\noindent
\begin{minipage}[c]{0.57\linewidth}
For any first-order method mentioned in Sect.~\ref{sec:related}, the NN update
\(\LogitD(\vx) - \LogitI(\vx)\) cannot be quantified in terms of
unlearning efficacy: it is a function of the optimization trajectory
and carries no information about the ideal classifier $\IdealR\sach{y}{\vx}$.
Our key idea is replacing it directly by the proxy shift
\(\log \Proxy\sachq{r}{\cdot}{\vx} - \log \Proxy \sach{\cdot}{\vx}\), the log-ratio of the two proxy posteriors, for which the shift is fully interpretable.
Full details in App.~\ref{app:why}.

\end{minipage}
\hfill
\begin{minipage}[c]{0.40\linewidth}
    \includegraphics[width=0.9\linewidth,trim=0 0.7cm 0 1cm,clip]{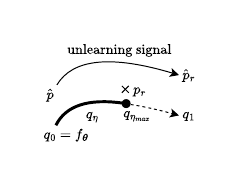}
    \captionof{figure}[]{The unlearning signal is computed from the proxies $(\Proxy, \ProxyR)$ and scaled by $\eta$ to produce the unlearned family.}
    \label{fig:main_method}
\end{minipage}

\paragraph{Unlearning with proxies.} We look for a distribution $\ProbitU\sach{\cdot}{\vx}$, which is an update similar to the initial model $\ProbitI(\vx)$, closer to the ideal retain distribution $\IdealR$ than the initial distribution $\Ideal$. This objective is framed in terms of KL divergences in the following optimization problem:
\begin{equation}
 \label{eq:main-underlying}
 \min_{\ProbitU}\; \kl[y\mid\vx]{\ProbitU}{\ProbitI}
 \quad\text{s.t.}\quad
 \kl[y\mid\vx]{\ProbitU}{\IdealR}\leq \kl[y\mid\vx]{\ProbitU}{\Ideal},\quad
 \sum_y \ProbitU\sach{y}{\vx}=1.
 \end{equation}
An interpretation is the following. Among all distributions $\ProbitU$ that fall on the $\IdealR$ side of the decision boundary (discriminating whether samples from $\ProbitU$ are more likely to be sampled from $\Ideal$ or $\IdealR$), we seek the one closest to the initial model in relative entropy $\ie$ the most $\Ideal$-like distribution that the Neyman-Pearson likelihood test would still attribute in expectation to $\IdealR$ instead of $\Ideal$.

Yet, $\Ideal$ and $\IdealR$ are unknown and cannot be used as such in the above problem~\eqref{eq:main-underlying}.
We thus introduce a pair of proxies $(\Proxy, \ProxyR)$ for these two distributions. The formulation can be rewritten as:
\begin{equation}
\label{eq:main-underlying-approx}
\min_{\ProbitU}\; \kl[y\mid\vx]{\ProbitU}{\ProbitI}
\quad\text{s.t.}\quad
\kl[y\mid\vx]{\ProbitU}{\ProxyR}\leq \kl[y\mid\vx]{\ProbitU}{\Proxy},\quad
\sum_y \ProbitU\sach{y}{\vx}=1.
\end{equation}
The interpretation is the same as before, but now the decision boundary is defined by the proxies $(\Proxy, \ProxyR)$ instead of the true distributions $(\Ideal, \IdealR)$.
One family that follows from it is, for a constant $\eta \geq 0$ and some normalization constant $\gammaU(\vx;\eta)$:
\begin{equation}
    \LogitU\sachq{\eta}{\cdot}{\vx} \doteq \LogitI(\vx) + \eta\,\left(\log\Proxy\sachq{r}{\cdot}{\vx} - \log\Proxy\sach{\cdot}{\vx}\right) - \gammaU(\vx;\eta)\1 .
    \label{eq:delta-M}
\end{equation}

\paragraph{Proxy design.} Sect.~\ref{sec:lda2c} introduces some proxy families.
However, we need that the model $\ProxyR$ is a better proxy for $\IdealR$ than $\Proxy$ is (see Prop.~\ref{prop:decrease}).
Although this assumption seems reasonable, Appendix~\ref{app:check_prop1} studies the condition and introduces an empirical safeguard checking at least its retain part.

\paragraph{Adaptive scaling.} Once the proxies pass the empirical safeguard, we find the maximal value $\eta_{\max}$ as the upper bound of a theoretically safe regime (detailed in App.~\ref{par:normalization}).
It is the largest scale at which the unlearning signal can be amplified while ensuring the theoretical bounds given by Prop.~\ref{prop:decrease} and~\ref{prop:tight-bound}.
In practice, no gradient nor back-propagation are required:
$\eta_{\max}$ is computed as the next zero of a convex scalar function by $1$-d bisection.

\paragraph{Optional distillation.} The final proposed unlearned model is then this precise $\ProbitU_{\eta_{\max}}$. The update $\eta_{\max}\ProxyShift$ may be implemented as a logit processor on top of the initial classifier $\LogitI$ taken as a black-box provided it outputs a logit vector.
An alternative is to distill the unlearning signal in the NN parameters $\vtheta$ with the standard teacher-student approach~\citep{hinton2015distillingknowledgeneuralnetwork}:
\begin{equation}
    \vtheta_u = \vtheta + \underset{\delta \vtheta}{\arg\min}\,
    \E\kl[y\mid\vx]{\ProbitU_{\eta_{\max}}}{f_{\vtheta + \delta\vtheta}}.
    \label{eq:lda2c-lstsq-grad}
\end{equation}

\subsection{Theoretical guarantees for well designed proxies}
\label{sec:thereortical}

This section gives a sketch of the reasoning, which is fully detailed in App.~\ref{app:justification} and~\ref{app:uphlimitations}.

The main quantity of the study, $\E\kl[y\mid\vx]{\IdealR}{\ProbitU_\eta}$, gauges how close the probits $\ProbitU_\eta$ of the unlearned model are from the ideal distribution given by the conditional true distribution $\IdealR\sach{y}{\vx}$.
It can be exactly written as (see App.~\ref{app:Decrease}):
\begin{align}
    \label{eq:decomposition}
        \underbrace{\E\kl[y\mid\vx]{\IdealR}{\ProbitI}}_{\text{initial divergence}}
        + \underbrace{\Deltagamma(\eta)}_{\leq 0, \forall \eta\in[0,\eta_{\max}]}
        + \eta \left(\underbrace{\E\kl[y\mid\vx]{\IdealR}{\ProxyR} - \E\kl[y\mid\vx]{\IdealR}{\Proxy}}_{\leq 0\text{ if $\ProxyR$ models better $\IdealR$ than $\Proxy$ does}}\right).
\end{align}
The derivation involves a specific normalizing term $\Deltagamma(\eta)$, which is a convex function of $\eta$ with $\Deltagamma(0)=0$. The value $\eta_{\max}$ is defined as the next zero of this function so that $\Deltagamma(\eta)$
is negative in the safety range $[0,\eta_{\max}]$ by design (see App.~\ref{par:normalization}).
Finding $\eta_{\max}$ is a simple 1-d line search over the empirical evaluation of the convex function $\eta \mapsto \Deltagamma(\eta)$ (see Alg.~\ref{alg:etamax} in App.~\ref{par:normalization}).
The third term is negative if $\ProxyR$ is a better proxy for $\IdealR$ than $\Proxy$ is.

\begin{proposition}[Conditional decrease of KL for the unlearned family]
    \label{prop:decrease}
    For any $\eta\in[0,\eta_{\max}]$, if we assume that
    proxies are well designed $\ie$ $\E\kl[y\mid\vx]{\IdealR}{\ProxyR} \leq \E\kl[y\mid\vx]{\IdealR}{\Proxy}$ then
    \begin{equation}
        \underbrace{\E\kl[y\mid\vx]{\IdealR}{\ProbitU_\eta}}_{\text{unlearned}} \leq \underbrace{\E\kl[y\mid\vx]{\IdealR}{\ProbitI}}_{\text{initial divergence}}.
    \end{equation}
\end{proposition}
This proposition is a soundness statement outlining a sufficient assumption for a decrease.
It asks $\ProxyR$ to track $\IdealR$ better than what $\Proxy$ does in expectation.
This seems reasonable at first sight. Yet, experimental results in App.~\ref{app:check_prop1} show that this is not always the case especially when the quality of the proxies is weak or when the features are less informative.
This is the reason why Alg.~\ref{alg:etamax} first empirically checks whether the retain part of the assumption indeed holds true with the test proposed in App.~\ref{app:check_prop1}.
If not, it sets $\eta_{\max}$ to 0, which is a conservative way to abort the unlearning while flagging the proxies.

In contrast, Prop.~\ref{prop:tight-bound} does not require this assumption, but it involves the training error (see App~\ref{App:Prop2}).

\begin{proposition}[Bound on the divergence under small modeling shift]
\label{prop:tight-bound}
For any $\eta\in[0,\eta_{\max}]$, the main quantity \(\E\kl[y\mid\vx]{\IdealR}{\ProbitU_\eta}\) is tightly bounded by:
\begin{align}
        (1-\eta)\underbrace{\E\kl[y\mid\vx]{\IdealR}{\ProbitI}}_{\text{initial divergence}}
        + \eta\left(\underbrace{\E\kl[y\mid\vx]{\IdealR}{\ProxyR} - \E\kl[y\mid\vx]{\Ideal}{\Proxy}}_{y\mid\vx\text{ modeling shift}}\right) + \eta\left(\underbrace{\E\kl[y\mid\vx]{\Ideal}{\ProbitI}}_{\text{training error}} + \gO {(\pi_f) }\right),
        \label{eq:tight-bound}
    \end{align}
where every $\E$ is taken under $\vx\sim\Ideal_\vx$ and
$\gO(\pi_f)$ is the explicit residual~\eqref{eq:Opif}.\end{proposition}

The first term of the bound is close to zero either when the posteriors of $\IdealR$ and $\Ideal$ are similar (nothing to unlearn), or when $\eta$ is close to $1$.
The second term of the bound can be small even if the proxies $(\ProxyR,\Proxy)$ may not accurately model their respective targets $(\IdealR,\Ideal)$.
A rule of thumb is to model both $\Ideal$ and $\IdealR$ similarly so their KL divergences compensate.
Modeling only the retrain set is sub-optimal. The third term is the initial model's training error, which the unlearner cannot control.
Inequality~\eqref{eq:tight-bound} is our tightest bound. It becomes an equality for $\eta=\eta_{\max}$ (see App.~\ref{App:Prop2}).
Section~\ref{sec:experiments} sets $\eta=\eta_{\max}$ in every experiment to maximally reduce this bound.
Of note, $\eta_{\max}$ might be null when the proxies are not sufficiently accurate or when the difference between $(\IdealR,\Ideal)$ is so small that there is nothing to unlearn, and the network remains untouched.

\section{Construction of unlearning-aware proxies}
\label{sec:lda2c}

This section proposes modeling the data distribution with a proxy $\Proxy\sach{\vx}{y}$, \ie\ a parametric mathematical model of the true class-conditional distribution. This proxy induces the following posterior, whose logarithm plays the role of a `logit' vector of $\R^C$:
\begin{equation}
\label{eq:posterior}
     \forall y\in \llbracket C\rrbracket, \forall \vx\in\sX, \quad \Proxy\sach{y}{\vx}
    \doteq \frac{\Proxy\sach{\vx}{y}\Ideal(y)}{\sum_{c=1}^C \Proxy\sach{\vx}{c}\Ideal(c)}.
\end{equation}

Appendix~\ref{app:Prop3} pursues the theoretical derivation for the new class-conditional data models $\vx\mid y$.
\begin{proposition}[Bound in terms of class-conditional data modeling errors]
\label{thm:mainTHM}
Regarding Prop.~\ref{prop:tight-bound}, the $y\mid\vx$ modeling shift term, $\E\kl[y\mid\vx]{\IdealR}{\ProxyR} - \E\kl[y\mid\vx]{\Ideal}{\Proxy}$, can be rewritten as
    \begin{align}
\underbrace{\E_r\kl[\vx\mid y]{\IdealR}{\ProxyR} - \E\kl[\vx\mid y]{\Ideal}{\Proxy}}_{\text{data conditional modeling shift}} +
\underbrace{\kl[\vx]{\Ideal}{\Proxy} - \kl[\vx]{\Ideal_{r}}{\Proxy_{r}}}_{\text{raw data modeling shift}} + \gO{(\pi_f)}.
    \end{align}
\end{proposition}
This highlights the advantage of using similar proxy data models for both sides of the unlearning signal, as already outlined in Prop~\ref{prop:tight-bound}.

\subsection{The path to the 2C construction}
\label{sub:2c}
The \emph{naïve} application of this approach is to fit one class-conditional model per class, $\Proxy$ on $\train$ and $\ProxyR$ on $\train_r$.
A refinement assumes the following \emph{mixture}:
for each class $y$,
\(
     \Proxy\sach{\vx}{y} = (1-\pi_f(y)) \Proxy\sachq{r}{\vx}{y} + \pi_f(y) \Proxy\sachq{f}{\vx}{y},
\)
where $\pi_f(y) = \nicefrac{|\train_f(y)|}{|\train(y)|}$ is the true intra-class proportion of the forget set.

\paragraph{The 2C lift.} This refinement
artificially considers the variable $(y,s)\in\llbracket C \rrbracket\times\{r,f\}$ as a label itself. This results in a mixture of $2C$ distributions. As a result, the proxy on the initial data intrinsically estimates $\Proxy\sach{s}{\vx}$.
This helps unlearning, especially when the forget data are not independent.
To our knowledge, no other work proposes this lifting.
This subtle shift results in a final mixture on the posterior rather than on $\vx\mid y$. Alg.~\ref{alg:modeling} details the mapping back to dimension $C$ to create the final proxies.
Its main advantages are : i) this single fit over the $2C$ dimensions returns both proxies $\ProxyR$ and $\Proxy$ ii) with similar accuracy to leverage the modeling shift compensation in Prop.~\ref{prop:tight-bound} and Prop.~\ref{thm:mainTHM}.

\begin{algorithm}[t]
\caption{\textsc{Modeling} with $2C$ lift. Pair
$(\Proxy,\ProxyR)$ obtained from a single fit on labels $(y,s)$.}
\label{alg:modeling}
\begin{algorithmic}[1]
\Require forget set $\train_f$, retain set $\train_r$, feature map $\phi$, any class-conditional family
\State $\vz \gets \phi(\vx)$ for every $\vx\in\train$ \Comment{proxies may live in feature space}
\State $s(\vx) \gets r$ if $\vx\in\train_r$, else $f$ if $\vx\in\train_f$ \Comment{\textbf{2C lift}: labels are now $(y,s)\in\llbracket C\rrbracket\times\{r,f\}$}
\ForAll{$(y,s)\in\llbracket C\rrbracket\times\{r,f\}$}
    \State $\pi_s(y) \gets \nicefrac{|\train_s(y)|}{|\train(y)|}$ \Comment{subclass prior}
    \State data modeling $\Proxy\sach{\vz}{y,s}$
\EndFor
\State $\Proxy\sach{y,s}{\vz} \gets$ Bayes inversion~\eqref{eq:posterior} \Comment{posterior over the $2C$ labels}
\State $\Proxy\sach{y}{\vz} \gets \Proxy\sach{y,r}{\vz} + \Proxy\sach{y,f}{\vz}$ \Comment{sum back to dimension $C$}
\State $\Proxy\sachq{r}{y}{\vz} \gets \nicefrac{\Proxy\sach{y,r}{\vz}}{\sum_{c=1}^C \Proxy\sach{y,r}{\vz}}$ \Comment{condition, do not refit}
\State \Return $(\Proxy,\ProxyR)$
\end{algorithmic}
\end{algorithm}

This paradigm applies to any data model. Alg.~\ref{alg:modeling} only requires a family that
can be fitted per $(y,s)$ and inverted by the Bayes rule~\eqref{eq:posterior}.
The Gaussian models of Sect.~\ref{sub:gaussian} and the empirical measure of Sect.~\ref{sub:dirac} are two examples among mixtures, kernel densities, or any normalized density model. Another trick is to model the distribution of the features $\phi(\vx)$ rather than the raw input data $\vx$. The feature map $\phi$ can be a frozen backbone or a projection onto a small subspace.

Section~\ref{sec:experiments} reports results with parametric models because they make the procedure efficient and interpretable with no gradient nor hyper-parameter. Yet, this approach also applies to a learned $2C$ head on top of the feature map $\phi$, reading
$\Proxy\sach{y,s}{\phi(\vx)}$ off its softmax.

\subsection{Application to Gaussian proxies}
\label{sub:gaussian}
Table~\ref{tab:proxy-zoo} summarizes the three constructions with Gaussian
models of $\vx\mid y$ and links it to known discriminators $\LDAI$ and $\QDAI$ (see App.~\ref{app:homoscedastic}).
All of them need the same $2C$ estimations of expectation vectors, but they differ
on the number of covariance matrix estimates.
\LDAIJ\ applies a unique $\LDAI$  on $(y,s)\in\llbracket C \rrbracket\times\{r,f\}$.
\QDAI\ remains unaffected by double labels.

\begin{table}[t]
  \centering
  \small
  \caption{Proxies based on the Gaussian distribution.
  Columns count the number of
  expectation vectors and covariance matrices to estimate the pair
  $(\Proxy,\ProxyR)$.}
  \label{tab:proxy-zoo}
  \begin{tabular}{llccc}
    \toprule
     & initial proxy $\Proxy$ & retain proxy $\ProxyR$ & \#$\vmu$ & \#$\mSigma$ \\
    \midrule
    \multicolumn{5}{l}{\textit{naive : the forget set is never looked at}} \\
    \LDAI   & $\gN(\vmu_{y};\mSigma)$   & $\gN(\vmu_{r,y};\mSigma_r)$   & $2C$ & $2$  \\
    \QDAI   & $\gN(\vmu_{y};\mSigma)$ & $\gN(\vmu_{r,y};\mSigma_{r,y})$ & $2C$ & $2C$  \\
    \addlinespace
    \multicolumn{5}{l}{\textit{mixture : $\Proxy$ is a mixture of a retain and a forget fit}} \\
    \LDAII  & with $\ProxyF\!:\gN(\vmu_{f,y};\mSigma_f)$     & $\gN(\vmu_{r,y};\mSigma_r)$     & $2C$ & $2$  \\
    \QDAII  & with $\ProxyF\!:\gN(\vmu_{f,y};\mSigma_{f,y})$ & $\gN(\vmu_{r,y};\mSigma_{r,y})$ & $2C$ & $2C$ \\
    \addlinespace
    \multicolumn{5}{l}{\textit{$2C$ lift : one fit per label $(y,s)$, mixture at the posterior level}} \\
    \LDAIJ  & $\gN(\vmu_{s,y};\mSigma)$ & \textit{conditioning} & $2C$ & $1$ \\
    \bottomrule
  \end{tabular}
\end{table}

\subsection{Application to the empirical measure leads to an $\eta$-free baseline}
\label{sub:dirac}

Replacing the Gaussian distribution by the empirical measure ignores the features. The scale $\eta$ disappears because $\log\ProxyR - \log\Proxy$ takes values in $\{0,-\infty\}$, so any $\eta>0$ gives the same target.

\paragraph{The empirical measure.} \DiracI\ is the naive approach applied to the empirical measure, \ie\ a sum of Diracs:
\begin{equation}
    \label{eq:DiracI}
    \Proxy: (\vx,y) \sim \train, \quad \ProxyR: (\vx,y) \sim \train_r.
\end{equation}
Appendix~\ref{app:exact_dirac} shows that this creates an unlearned model
s.t. $\ProbitU(\vx) = \ProbitI(\vx)$ if $\vx\notin\train_f$, otherwise
\begin{equation}
    {\ProbitU(\vx)}_c =
    \begin{cases}
        0 & c = y(\vx), \\
        \ProbitI(\vx)_c\,/\,(1-\ProbitI(\vx)_{y(\vx)}) & c \neq y(\vx).
    \end{cases}
    \label{eq:dd_target_softmax}
\end{equation}
Its main use is to be a legitimate teacher for the upcoming distillation in the case of a class-wise forgetting scenario due to the fact that it exactly annihilates the probit of $y(\vx)$ on the forget set. However, this is not suitable for a subclass forgetting scenario.

\paragraph{$2C$ lifting for Diracs.} \DiracIJ\ is the $2C$ lift of Sect.~\ref{sub:2c} applied to that
same empirical measure:
\begin{equation}
    \label{eq:DiracIJ}
    \Proxy: (\vx,y,s) \sim \train\times \{r,f\}, \quad \ProxyR: y\mid\vx \sim \Proxy\sach{y}{\vx,r}.
\end{equation}
Its state posterior being feature agnostic, $\Proxy\sach{s}{\vx,y}\simeq\pi_s(y)$ (see App.~\ref{app:exact_dirac_2c}), the pair reduces to empirical counts. The same derivation leaves $\ProbitU(\vx)=\ProbitI(\vx)$ outside
$\train_f$ and, on the forget set, gives:
\begin{equation}
    {\ProbitU(\vx)}_c = \frac{|\train_r(y(\vx))|}{|\train(y(\vx))|} \ProbitI(\vx)_c + \frac{|\train_f(y(\vx))|}{|\train(y(\vx))|}
    \begin{cases}
        0 & c = y(\vx), \\
        \ProbitI(\vx)_c\,/\,(1-\ProbitI(\vx)_{y(\vx)}) & c \neq y(\vx),
    \end{cases}
    \label{eq:dd2c_target_softmax}
\end{equation}
where $|\train(y(\vx))|$ is the number of training samples in $\train$ pertaining to the same class $y(\vx)$ than sample $\vx$.
Compared to~\eqref{eq:dd_target_softmax}, the same probit is now partially cancelled, in proportion of the forget set.
This adapts \DiracI\ in the specific class-wise unlearning scenario. This is specifically what an agnostic unlearner would expect for behavior.
Of note, \DiracIJ\ is the cheapest proxy of the library, and the only one oblivious to the retain dataset. It only needs the count numbers $\{|\train_r(y)|\}_y$ and $\{|\train(y)|\}_y$.

\section{Experiments}
\label{sec:experiments}

\subsection{Experimental protocol}
\textbf{Datasets and architectures.}
The explored datasets are CIFAR-10 and  CIFAR-100~\citep{Krizhevsky2009cifar}.
We consider (3+1) architectures.
The first three use \textbf{DINOv2-Small}~\citep{oquab2023dinov2} as a frozen backbone foundation model, completed with
the feed forward heads: \textbf{linear} ($384\to C$),
\textbf{mlp1} ($384\to 256\to C$),
or \textbf{mlp2} ($384\to 256 \to 256\to C$).
The fourth is \textbf{ResNet-18}, where we apply our method to internal projected features (see App.~\ref{app:random_projection}).

\textbf{Unlearning methods.}
The results contain \LDAIJ, the distilled \DiracI\ and its variant \DiracIJ\ ,
\QDAI\ and \QDAII\ with diagonal covariances, \LDAI\ and \LDAII\ (see App.~\ref{app:homoscedastic}).
Our method is evaluated against iterative baselines trained for at most 20 epochs:
\textbf{FT} (retain fine-tuning),
\textbf{GA} (gradient ascent), \textbf{GA+FT} (one ascent followed by fine-tuning),
\textbf{RL+FT}~\citep{graves2021amnesiac} (random labelling the forget set),
\textbf{SCRUB}~\citep{Kurmanji2023scrub} (min-max teacher-student),
\textbf{SalUn}~\citep{Fan2024salun} (gradient ascent masking and RL+FT).

\textbf{Unlearning scenarios.}
\emph{\class} where $\train_f$ is an entire class.
\emph{\subclass} for CIFAR-10 consists in training a binary classifier (vehicle vs. animal) and forgetting one subclass per class at a time. For CIFAR-100 it is $C=20$ classes, each composed of five subclasses.
This scenario is difficult since the forget set is structured within the class. This causes gradient ascent baselines to fail~\citet{mavrothalassitis2026ascent}.

\textbf{Metrics.} $\KL^\TEST$ and $\Acc^{\TEST}$ (resp. $\KL^\FORGET$ and $\Acc^{\FORGET}$) measure $\kl[y|\vx]{\ProbitR}{q}$ and the accuracy empirically over the test set (resp. forget set).
We also define the \emph{noise floor} as the average $\kl[y|\vx]{\ProbitR}{\ProbitR^\prime}$ between two independently retrained models.
This helps gauge how close an unlearned model is to models retrained from scratch.
Per-sample membership inference audit with relative test, inspired by the work of~\cite{ribero2026regularizedfdivergencekerneltests}, as a percentage of data in the forget set whose outputs look more from retrain than from initial (see App.~\ref{sec:peraudit}). It is an empirical diagnostic, not a privacy guarantee.
RTE is the runtime normalized by the duration of training $\ProbitR$ until it obtains $99\%$ training accuracy.

\subsection{Experimental validation}
\label{sec:exp_theory}
\begin{figure}[tb]
  \centering
  \includegraphics[width=\columnwidth]{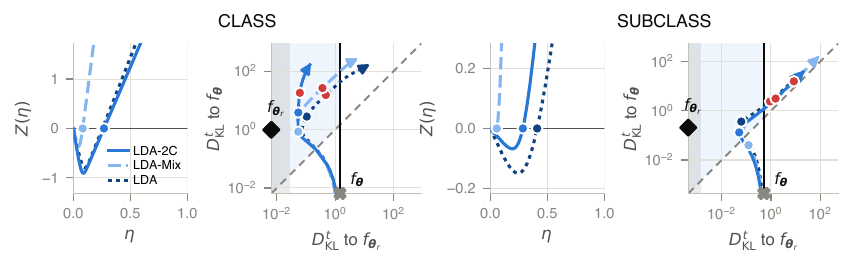}
  \caption{CIFAR-10, DINOv2 + MLP-1; \class\ and \subclass\ (airplane) scenarios. Red dots at $\eta=1$, blue dots are our adaptive $\eta_{\max}$.
  (left) $Z(\eta)$ for $\eta\in[0,1]$.
  (right) path of the target $\ProbitU_\eta$ in the
  plane $\left(\KL^{\TEST}(\ProbitR||\ProbitU_\eta),\KL^{\TEST}(\ProbitI||\ProbitU_\eta)\right)$.
  }
  \label{fig:h_eta}
\end{figure}

\textbf{About $\eta_{\max}$.}
Figure~\ref{fig:h_eta} (left) empirically verifies two statements about the $\Deltagamma$ function: convexity and uniqueness of the zero $\eta_{\max}$ in $[0,1)$. Alg.~\ref{alg:etamax} has no difficulty in finding this zero, although its value depends a lot on the proxies and on the unlearning scenario.
Figure~\ref{fig:h_eta}(right) shows the parametric plot  $\kl{\ProbitI}{\ProbitU_\eta}$ \textit{vs.} $\kl{\ProbitR}{\ProbitU_\eta}$ as $\eta$ goes from 0 to $\infty$.
Model $\ProbitU_\eta$ starts from $\ProbitI$, gets closer to $\ProbitR$, and then goes away from both models.
Along this trajectory, $\eta_{\max}$ almost always indicates the sweet spot closer to $\ProbitR$.
On the contrary, a scaling set to 1 by default (red dots) is farther away from both reference models $\ProbitI$ and $\ProbitR$.
This reveals that the scaling is a crucial parameter.

\textbf{Getting closer to the ideal.} Figure~\ref{fig:prop1validation} compares the architectures DINOv2 and ResNet-18 to outline the dependence on the features.
Our $\eta_{\max}$ construction almost always gets closer to the ideal model as promised by Prop.~\ref{prop:decrease}.
This is striking for DINOv2 but less tangible for ResNet-18, which provides less informative features.
However, our targets are mostly on the edge of the noise floor.
Competitors cannot achieve this goal with certainty.

\begin{figure}[tbh]
  \centering
  \includegraphics[width=\columnwidth, trim=0.4cm 1cm 0.4cm 0.2cm,clip]{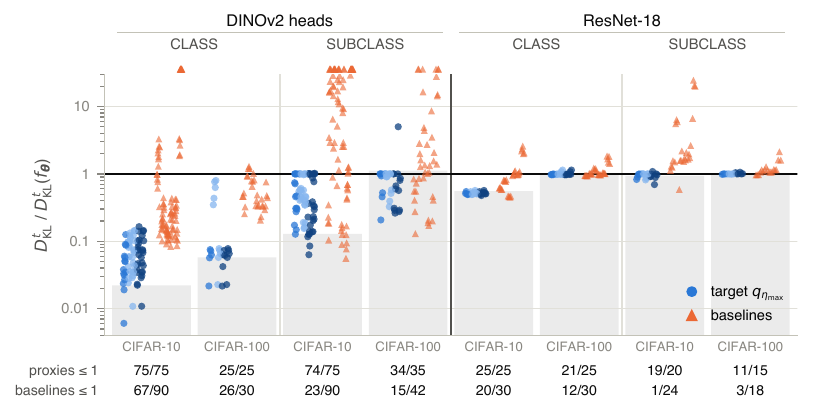}
  \caption{Empirical validation of Prop.~\ref{prop:decrease}.
  $\KL^{\TEST}/\KL^{\TEST}(\ProbitI)$ for 4 architectures, mean over 5 seeds.
  (blue): \LDAIJ, \LDAII, \QDAII, \LDAI, \QDAI; (orange): the six baselines at their last epoch.
  }
  \label{fig:prop1validation}
\end{figure}

\textbf{Comparison between proxies.}
Figure~\ref{fig:hero-evolution} illustrates the benefits of the 2C construction:
\DiracI\ destroys the entire class, and its distillation does not verify Prop.~\ref{prop:decrease}, while \DiracIJ\ does.
Figure~\ref{fig:accuracies} shows the drawback of \LDAII\, the only proxy that estimates a full covariance matrix for $\train_f$ alone.
On the contrary, \LDAI\ and \LDAIJ\ use a single pooled covariance matrix estimated across all classes.
These comparisons reveal the trade-off between the ease of proxies' parameter identification and the cost of modeling inaccuracies.
$\QDAI$ and $\QDAII$ are inaccurate quadratic proxies because we impose diagonal covariances.
This is the reason why \QDAI\ \emph{underperforms} \LDAI.
Appendix~\ref{app:homoscedastic} details the cost for unlearning due to the homoscedastic assumption, \ie\ full \QDAI\ versus \LDAI\ .

\subsection{Benchmark}
\label{sec:exp_benchmark}
\begin{figure}[tb]
  \centering
  \includegraphics[width=\columnwidth,trim=0 0.3cm 0 1.1cm, clip]{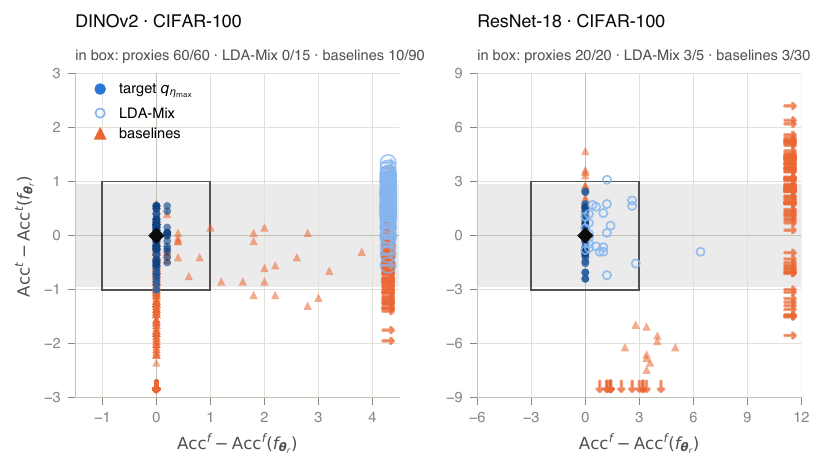}
  \caption{
  $\Acc_{\FORGET}-\Acc_{\FORGET}(\ProbitR)$ \textit{vs.} $\Acc_{\TEST}-\Acc_{\TEST}(\ProbitR)$, \class\ scenario.
  square: tolerance $\pm1$ pt (DINOv2 heads-left), $\pm3$ pt (ResNet-18-right);
  grey band: accuracy noise floor over the retrain seeds.
  }
  \label{fig:accuracies}
\end{figure}

\textbf{Behavioral advantage.}
Our proxies are competitive in the \emph{Class} scenario, matching the ideal model's accuracy exactly over the forget set and within the accuracy noise floor on the test set (see Fig.~\ref{fig:accuracies}).
As for \emph{Subclass}, they reach the lowest KL well below the baselines (see Fig.~\ref{fig:hero-evolution}) at a much smaller RTE (see Table.~\ref{tab:hero-benchmark}). On DINOv2, our approach is the best affordable unlearning procedure.

\begin{table}[t]
  \centering
  \footnotesize
  \setlength{\tabcolsep}{4pt}
  \caption{Every method gets the same budget, $10\%$ of the retraining, read at the last epoch (Ep.) it can afford. CIFAR-10 / DINOv2+\texttt{mlp1}, \subclass\ (airplane), mean$\pm$std over $5$ seeds. The targets $\ProbitU_{\eta_{\max}}$ take no gradient step so $\RTE$ is the proxy fit. \textbf{Bold}: closest to $\ProbitR$ in the columns.}
  \label{tab:hero-benchmark}
  \begin{tabular}{lccccccc}
    \toprule
     & $\Acc^{\FORGET}$ (\%) & $\Acc^{\TEST}$ (\%) & $\KL^{\FORGET}$ (nats) & $\KL^{\TEST}$ (nats) & Audit (\%) & $\RTE$ (\%) & Ep. \\
    \midrule
    $\ProbitI$     & $100.0$ & $99.8$ & $5.0$ & $0.49$ & $0.0$ & -- & -- \\
    $\ProbitR$     & $59.1$ & $95.8$ & $0.010$ & $1.1\cdot 10^{-3}$ & $100.0$ & $100$ & -- \\
    \midrule
    \LDAIJ         & $\boldsymbol{61.9}{\scriptstyle\,\pm\,0.1}$ & $\boldsymbol{96.3}{\scriptstyle\,\pm\,0.1}$ & $\boldsymbol{0.48}{\scriptstyle\,\pm\,0.02}$ & $\boldsymbol{0.062}{\scriptstyle\,\pm\,0.003}$ & $59.3{\scriptstyle\,\pm\,1.1}$ & $4.7$ & 0 \\
    \LDAII         & $92.3{\scriptstyle\,\pm\,0.1}$ & $99.1{\scriptstyle\,\pm\,0.1}$ & $1.7{\scriptstyle\,\pm\,0.1}$ & $0.20{\scriptstyle\,\pm\,0.01}$ & $14.2{\scriptstyle\,\pm\,0.7}$ & $3.0$ & 0 \\
    \LDAI          & $45.8{\scriptstyle\,\pm\,0.2}$ & $95.0{\scriptstyle\,\pm\,0.0}$ & $0.61{\scriptstyle\,\pm\,0.04}$ & $\boldsymbol{0.062}{\scriptstyle\,\pm\,0.004}$ & $\boldsymbol{72.6}{\scriptstyle\,\pm\,0.5}$ & $3.0$ & 0 \\
    \QDAII         & $69.9{\scriptstyle\,\pm\,0.1}$ & $97.2{\scriptstyle\,\pm\,0.1}$ & $0.63{\scriptstyle\,\pm\,0.02}$ & $0.079{\scriptstyle\,\pm\,0.003}$ & $49.3{\scriptstyle\,\pm\,1.3}$ & $3.3$ & 0 \\
    \QDAI          & $33.7{\scriptstyle\,\pm\,0.4}$ & $93.9{\scriptstyle\,\pm\,0.1}$ & $0.96{\scriptstyle\,\pm\,0.05}$ & $0.099{\scriptstyle\,\pm\,0.006}$ & $\boldsymbol{72.8}{\scriptstyle\,\pm\,0.5}$ & $3.0$ & 0 \\
    \midrule
    SCRUB          & $4.1{\scriptstyle\,\pm\,3.6}$ & $90.2{\scriptstyle\,\pm\,0.7}$ & $29{\scriptstyle\,\pm\,20}$ & $2.6{\scriptstyle\,\pm\,1.8}$ & $48.4{\scriptstyle\,\pm\,9.6}$ & $7.7$ & 3 \\
    GA{+}FT        & $6.8{\scriptstyle\,\pm\,1.4}$ & $90.8{\scriptstyle\,\pm\,0.2}$ & $25{\scriptstyle\,\pm\,4}$ & $2.1{\scriptstyle\,\pm\,0.4}$ & $44.9{\scriptstyle\,\pm\,1.5}$ & $7.8$ & 2 \\
    GA             & $0.0{\scriptstyle\,\pm\,0.0}$ & $60.8{\scriptstyle\,\pm\,0.0}$ & $(2.4{\scriptstyle\,\pm\,0.1})\cdot 10^{3}$ & $745{\scriptstyle\,\pm\,29}$ & $17.0{\scriptstyle\,\pm\,0.6}$ & $7.7$ & 20 \\
    SalUn          & $0.8{\scriptstyle\,\pm\,0.2}$ & $90.3{\scriptstyle\,\pm\,0.0}$ & $7.0{\scriptstyle\,\pm\,0.6}$ & $0.60{\scriptstyle\,\pm\,0.06}$ & $56.0{\scriptstyle\,\pm\,1.2}$ & $8.0$ & 2 \\
    RL{+}FT        & $1.2{\scriptstyle\,\pm\,0.1}$ & $90.3{\scriptstyle\,\pm\,0.0}$ & $9.5{\scriptstyle\,\pm\,0.7}$ & $0.81{\scriptstyle\,\pm\,0.06}$ & $53.6{\scriptstyle\,\pm\,1.1}$ & $9.5$ & 13 \\
    FT             & $95.7{\scriptstyle\,\pm\,2.3}$ & $99.5{\scriptstyle\,\pm\,0.2}$ & $2.7{\scriptstyle\,\pm\,0.4}$ & $0.27{\scriptstyle\,\pm\,0.05}$ & $6.3{\scriptstyle\,\pm\,3.4}$ & $7.7$ & 3 \\
    \bottomrule
  \end{tabular}
\end{table}

\textbf{Distillation and stopping criterion.}
Figure~\ref{fig:hero-evolution} shows that the distillation, which is an optional final step, slightly improves the behavior.
We explain this by the fact that the distilled $\ProbitD$ shares the same architecture as $\ProbitI$.
Distillation requires no rewind or stopping criterion; it converges quickly and remains stable.
In contrast, ascent-based methods GA and GA+FT deviate from expected behavior.
The others plateau or converge very slowly.

\begin{figure}[tb]
  \centering
  \includegraphics[width=0.9\columnwidth, trim=0.25cm 0.2cm 0.3cm 0.3cm, clip]{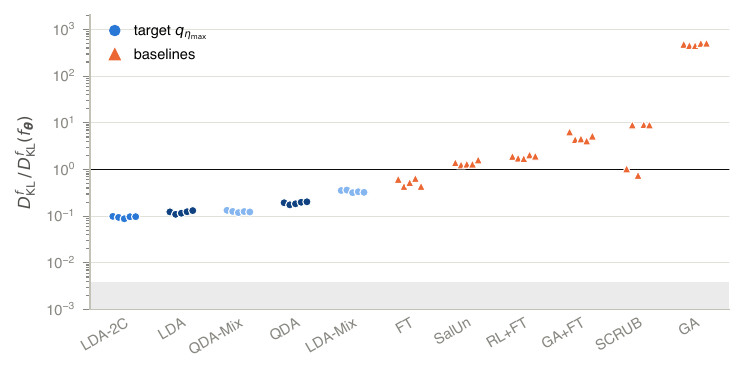}
  \includegraphics[width=0.9\columnwidth, trim=0.05cm 0.3cm 0.2cm 0.3cm, clip]{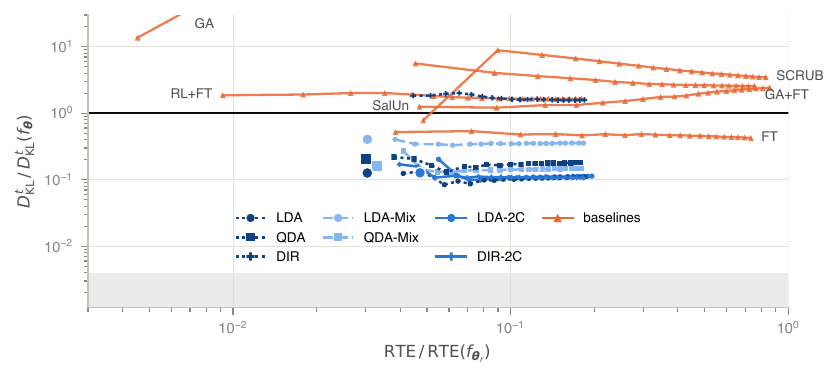}
  \caption{CIFAR-10, DINOv2 + MLP-1, \subclass\ (airplane).
  $\KL^{\TEST}/\KL^{\TEST}(\ProbitI)$. (top) final result for 5 runs; (bottom) averaged results along the unlearning epochs
   \textit{vs.} the cumulative $\RTE$.}
  \label{fig:hero-evolution}
\end{figure}

\section{Conclusion: advantages and limitations}
\label{sec:conclusion}

Our framework introduces a closed-form unlearning operator emerging from an explicit constrained optimization problem~\eqref{eq:main-underlying-approx}.
It derives the unlearning signal from any proxies, and adapts its scaling in a safe regime with $\eta_{\max}$ (see Alg.~\ref{alg:etamax}).
This approach can be seen as a logit processor on top of the initial black-box model.
Updating the weights by distillation is optional, fast, and stable.

Prop.~\ref{prop:decrease} and~\ref{prop:tight-bound} are general and apply to any previous or future proxy-based unlearning methods, while Prop.~\ref{thm:mainTHM} deals with data modeling.
We provide \emph{behavioral} guarantees that bound the divergence to the ideal distribution, not the retrained model, under an explicit design condition of the unlearning signal.
The experiments show that this indeed translates to the retrained model.

The paper instantiates this approach with a family of proxies without claiming optimality.
More useful are the general 2C construction and the dimension-reduction projection, which apply to any proxy.
If the proxies are detected as not good enough, the scaling is set to zero, leaving the model unchanged.
This signifies to the unlearner to use more refined proxy design.
In contrast, the previous approach without adaptive scaling may destroy the model by unlearning too much.

\newpage

\clearpage
\bibliographystyle{abbrvnat}
\bibliography{references}

\appendix
\appendix

\clearpage
\section{Glossary}
\label{app:notations}

\begin{table}[H]
\centering
\caption{List of the symbols with their short definition.}
\small
\begin{tabular}{p{0.24\textwidth} p{0.70\textwidth}}
\toprule
\textbf{Notation} & \textbf{Definition} \\
\midrule
$\vx, y$ & Input $\vx\in\sX$ and one-hot label $y\in \sY$\\
$\sX, \sY = \llbracket C\rrbracket$ & Input and labels space\\
$C$ & Output dimension: number of classes\\
$\train\subset \sX$ & Original training set sampled from $\Ideal$\\
$\train_r, \train_f$ & Retain and forget set $\train = \train_r \cup \train_f$\\
$\train(y)$, $\train_r(y)$, $\train_f(y)$ & Class-$y$ subsets of the corresponding sets\\
$\nf,\nr,\n$ & Sizes of sets $\train_f,\train_r,\train$\\
$\pi_r,\pi_f$ & Mixture weights of the retain/forget components in $\Ideal$, $\pi_r+\pi_f=1$\\
$s\in\{r,f\}$ & State variable that doubles the labels number in the $2C$ construction (Sect.~\ref{sec:lda2c})\\
$\pi_r(y),\pi_f(y)$ & Intra-class proportion $|\train_s(y)|/|\train(y)|$ used in the conditional mixture\\
\midrule
$\Ideal$ & True distribution of $(\vx,y)$\\
$\Ideal_\vx,\Ideal_y$ & Marginal true distributions\\
$\Ideal\sach{y}{\vx},\Ideal\sach{\vx}{y}$ & Conditional true distributions\\
$\IdealR$ & True distribution of the retain set\\
$\IdealF$ & True distribution of the forget set, $\Ideal = \pi_r\IdealR+\pi_f\IdealF$\\
\midrule
$\LogitI(\vx)$ & Logit vector output of the initial neural network at input $\vx$, of parameters $\vtheta$\\
$\LogitD(\vx)$ & Logit vector of the distilled NN, of parameters $\vtheta_u$\\
$\ProbitI(\vx), \ProbitR(\vx)$ & Predicted probabilities of the initial or retrained NN\\
$\ProbitU_\eta (\vx)$ & Target of the unlearned model: $\ProbitU_\eta \doteq\softmax(\LogitI + \eta\,\ProxyShift)$\\
$\ProbitD (\vx)$ & Probability vectors of the unlearned NN (after distillation)\\
\midrule
$\ProxyR\sach{y}{\vx},\Proxy\sach{y}{\vx}$ & Posteriors induced by the proxies\\
$\ProxyR\sach{\vx}{y}, \Proxy\sach{\vx}{y}$ & Proxy data models\\
$\Proxy\sachq{f}{\vx}{y}$ & Per-class proxy of the forget distribution used in the conditional mixture\\
\midrule
$\softmax$ & Standard softmax operator\\
$\lse(\vz)$ & Log-Sum-Exponential: $\lse(\vz)\doteq\log\sum_{i=1}^C e^{z_i}$\\
$\gamma(\vx)$ & Normalization constant of the initial NN: $\gamma(\vx)\doteq\lse(\LogitI(\vx))$\\
$\gammaU(\vx;\eta)$ & Normalization of the unlearned model: $\gammaU(\vx;\eta)\doteq\lse(\LogitU_\eta(\vx))$\\
$\Deltagamma(\vx;\eta)$ & Shift of the normalizations between the unlearned and the initial NN~\eqref{eq:Zdef}\\
$\eta\geq 0$ & Scaling parameter of the unlearning signal\\
$\eta_{\max}$ & Largest $\eta> 0$ such that $\E\Deltagamma(\vx;\eta) < 0$\\
$\1$ & All-one vector in $\R^C$\\
\midrule
$\kl[y\mid\vx]{p}{q}$ & $\sum_{y=1}^C p\sach{y}{\vx}\log\frac{p\sach{y}{\vx}}{q\sach{y}{\vx}}$. The conditional KL divergence\\
$\E\kl[y\mid\vx]{p}{q}$ & Its expectation over $\vx\sim\Ideal_\vx$. Main figure of merit of the paper (Sect.~\ref{sec:thereortical})\\
$\E\kl[\vx\mid y]{p}{q}$ & Symmetric expression conditioning on $y$ in the bound of Prop.~\ref{thm:mainTHM}\\
$\KL^{\TEST},\KL^{\FORGET}$ & Empirical expectation estimated on the test / forget set, as reported in the tables\\
\bottomrule
\end{tabular}
\label{tab:glossaire}
\end{table}

\section{Justifications}
\label{app:justification}

We assume that the data are an independent sampling of the true distributions: $\train\sim\Ideal^{\mid \train\mid }$ and $\train_r\sim\IdealR^{\mid \train_r\mid }$. The original true distribution is a mixture of form~\citep{Le2025pour}:
\begin{equation}
\Ideal = \pi_r\IdealR  + \pi_f\IdealF \quad\text{with}\quad \pi_r+\pi_f=1.
\label{eq:Amix}
\end{equation}

\subsection{The unlearned model family}
\label{app:why}

We look for a model $\ProbitU$, which is an update similar to the initial NN $\ProbitI$, closer to the ideal retain distribution $\IdealR$ than the initial distribution $\Ideal$.
This objective is framed in terms of KL divergences in the following optimization problem:
\begin{equation}
\label{eq:underlying}
\min_{\ProbitU}\; \E\kl[y\mid\vx]{\ProbitU}{\ProbitI}
\quad\text{s.t.}\quad
\E\kl[y\mid\vx]{\ProbitU}{\IdealR}\leq \E\kl[y\mid\vx]{\ProbitU}{\Ideal},\quad
\E\sum_y \ProbitU\sach{y}{\vx}=1.
\end{equation}

By definition of the $\KL$, one always has the identity
\begin{equation*}
    \E\kl[y\mid\vx]{\ProbitU}{\Ideal}-\E\kl[y\mid\vx]{\ProbitU}{\IdealR}
=\E\sum_y \ProbitU\sach{y}{\vx}\log\frac{\Ideal\sachq{r}{y}{\vx}}{\Ideal\sach{y}{\vx}}.
\end{equation*}
Following~\citep[Chap. 12]{cover2006elements}, we write the Lagrangian formulation of problem~\eqref{eq:underlying} with multipliers $\eta\geq 0$ and $\nu$:
\begin{equation*}
J(\ProbitU;\eta,\nu)\doteq\E\sum_y \ProbitU\sach{y}{\vx}\log\frac{\ProbitU\sach{y}{\vx}}{\ProbitI\sach{y}{\vx}}
   -\eta\E\sum_y \ProbitU\sach{y}{\vx}\log\frac{\Ideal\sachq{r}{y}{\vx}}{\Ideal\sach{y}{\vx}}
   +\nu\E\sum_y \ProbitU\sach{y}{\vx}.
\end{equation*}
Cancelling the derivative with respect to $\ProbitU\sach{y}{\vx}$, we obtain: $\forall y\in\sY$
\begin{equation}
\E\log\frac{\ProbitU\sach{y}{\vx}}{\ProbitI\sach{y}{\vx}}+1-\eta\E\log\frac{\Ideal\sachq{r}{y}{\vx}}{\Ideal\sach{y}{\vx}}+\nu=0.
\end{equation}

Rearranging the terms, the final expression can be expressed as:
\begin{equation}
\label{eq:expected-solution}
    \E\left[\log(\ProbitU\sachq{\eta}{\cdot}{\vx})\right] =
    \E\left[\underbrace{\log \ProbitI\sach{\cdot}{\vx}}_{\LogitI(\vx) - \gamma(\vx)\1} + \eta \underbrace{\log \frac{\Ideal\sachq{r}{\cdot}{\vx}}{\Ideal\sach{\cdot}{\vx}}}_{\text{exact unlearning signal }\log \IdealR (\cdot\mid\vx)- \log \Ideal (\cdot\mid\vx)} - \underbrace{\Deltagamma(\vx;\eta)}_{\text{normalizing term}}\right].
\end{equation}
\paragraph{Remark.} In~\citep[Chap. 12]{cover2006elements}, $\eta = \eta_{0}$ is chosen
so that:
\begin{equation*}
    \E\kl[y\mid\vx]{\ProbitU_{\eta_{0}}}{\Ideal}-\E\kl[y\mid\vx]{\ProbitU_{\eta_{0}}}{\IdealR}=0.
\end{equation*}
An attacker observing samples from $\ProbitU$ is not, in expectation, in capacity to decide whether the drawn samples are coming from $\IdealR$ or $\Ideal$.

\paragraph{Estimation.} One solution is to impose that~\eqref{eq:expected-solution} holds for any $\vx$.
This translates in the logit space to the solution:
\begin{equation*}
\label{eq:our-solution}
    \LogitI(\vx)+ \eta (\log \IdealR (\cdot\mid\vx)- \log \Ideal (\cdot\mid\vx)),
\end{equation*}
$\ie$ the family of unlearned model presented in Sect.~\ref{sec:our_method}.
The true exact unlearning signal $\log \IdealR (\cdot\mid\vx)- \log \Ideal (\cdot\mid\vx)$ being inaccessible in practice, our key idea consists in relying on proxies $(\Proxy, \ProxyR)$:
\begin{equation*}
    \log \ProxyR (\cdot\mid\vx)- \log \Proxy (\cdot\mid\vx) \approx \log \IdealR (\cdot\mid\vx)- \log \Ideal (\cdot\mid\vx).
\end{equation*}
Rephrasing, given proxies $(\Proxy, \ProxyR)$ we obtain our unlearned model family of type
\begin{equation*}
\label{eq:our-solution-approx}
    \LogitI(\vx)+ \eta (\log \ProxyR (\cdot\mid\vx) - \log \Proxy (\cdot\mid\vx)).
\end{equation*}
More precisely, it is the unbiased solution of the following optimization problem:
\begin{equation*}
\label{eq:underlying-approx}
\min_{\ProbitU}\; \E\kl[y\mid\vx]{\ProbitU}{\ProbitI}
\quad\text{s.t.}\quad
\E\kl[y\mid\vx]{\ProbitU}{\ProxyR}\leq \E\kl[y\mid\vx]{\ProbitU}{\Proxy},\quad
\E\sum_y \ProbitU\sach{y}{\vx}=1.
\end{equation*}

To conclude, we consider the following family of unlearned models:
\begin{equation}
    \LogitU\sachq{\eta}{\cdot}{\vx} \doteq \LogitI(\vx) + \eta\,\ProxyShift\sach{\cdot}{\vx} -\gammaU(\vx;\eta)\1,
    \label{eq:AppUnlearn}
\end{equation}
for some normalization constant $\gammaU(\vx;\eta)\doteq \lse (\LogitI(\vx) + \eta\,\ProxyShift\sach{\cdot}{\vx})$.
The logits and probits are related by:
\begin{equation}
    \ProbitU_\eta(\vx) \doteq \softmax\left(\LogitU_\eta(\vx)\right) = \softmax\left(\LogitI(\vx) + \eta\,\ProxyShift\sach{\cdot}{\vx}\right).
    \label{eq:AppSoftmax}
\end{equation}

\subsection{The safety range of \texorpdfstring{$\eta$}{eta}}
\label{par:normalization}

We first relate the normalization $\gammaU(\vx;\eta)$ to $\gamma(\vx)$, the normalization of the initial NN.

$\lse:\R^C\to\R$ denotes the Log-Sum-Exponential function defined as $\lse(\vz) \doteq \log\left(\sum_{i=1}^C e^{z_i}\right)$. It verifies, $\forall \vz \in\R^C, \forall \kappa \in\R$:
\begin{align}
    \log(\softmax (\vz)) &= \vz - \lse(\vz)\1\label{eq:LSE1}\\
    \lse(\vz + \kappa\1) &= \lse(\vz) + \kappa\label{eq:LSE2}\\
    \lse(\log(\vz)) &= \log\left(\sum_{i=1}^C z_i\right),\label{eq:LSE3}
\end{align}
where $\log(\vz)$ is an abuse of notation defining  the  vector $(\log z_1,\ldots, \log z_C)$ for $\vz\in\R^C_{>0}$, and $\1\in\R^C$ is the all-one vector.

Indeed, applying $\lse$ to the logits of the unlearned model isolates the normalization term due to~\eqref{eq:LSE2} and~\eqref{eq:LSE3}. Since we forge the unlearned model by adding the proxy shift to the initial NN~\eqref{eq:delta-M}, we have:
\begin{align*}
    \gammaU(\vx;\eta)
    &\stackrel{\eqref{eq:AppUnlearn}}{=} \lse \left( \LogitI(\vx) + \eta\log \left(\frac{\Proxy\sachq{r}{\cdot}{\vx}}{\Proxy\sach{\cdot}{\vx}}\right)\right)\\
    &\stackrel{\eqref{eq:AppSoftmax}}{=} \lse \left(\log \ProbitI(\vx)+\gamma(\vx)\1+ \eta\log \left(\frac{\Proxy\sachq{r}{\cdot}{\vx}}{\Proxy\sach{\cdot}{\vx}} \right)\right)\\
    &\stackrel{\eqref{eq:LSE2}}{=} \gamma(\vx) + \lse \left(\log \left(\frac{\ProbitI(\vx)\Proxy\sachq{r}{\cdot}{\vx}^\eta}{\Proxy^\eta\sach{\cdot}{\vx}}\right)  \right),
\end{align*}
with component-wise multiplication and division between vectors.
This links the normalization offset $\gammaU$ of the unlearned model to $\gamma$, the one of the initial model.
Thanks to~\eqref{eq:LSE3}, we have:
\begin{equation}
\label{eq:Zdef}
\Deltagamma(\vx;\eta)
    \doteq \gammaU(\vx;\eta) - \gamma(\vx)
    = \log \left( \sum_{y=1}^C \ProbitI\sach{y}{\vx} \left(\frac{\Proxy\sachq{r}{y}{\vx}}{\Proxy\sach{y}{\vx}}\right)^\eta \right)
\end{equation}

Note that $\Deltagamma(\vx;0) = 0, \,\forall \vx$, which is expected as the unlearning operator is void at $\eta=0$. Its derivative w.r.t. $\eta$ equals
\begin{align*}
    \frac{\partial\Deltagamma(\vx;\eta)}{\partial \eta} =
    \frac{\E_{y\sim\ProbitI\sach{y}{\vx}} \left[A^\eta(y,\vx)\log\left(A(y,\vx)\right)\right]}{\E_{y\sim\ProbitI\sach{y}{\vx}}\left[A^\eta(y,\vx)\right]},
\end{align*}
with $A(y,\vx)\doteq \nicefrac{\Proxy\sachq{r}{y}{\vx}}{\Proxy\sach{y}{\vx}}$, and because the derivative of $a^u=\exp(u\log(a))$ equals $a^u \log(a), \,\forall a>0$. Note that, at point $\eta=0$,
\begin{align*}
    \left.\frac{\partial\Deltagamma(\vx;\eta)}{\partial \eta}\right|_{\eta=0}
    &= \E_{y\sim\ProbitI\sach{y}{\vx}} \left[\log\left(\frac{\Proxy\sachq{r}{y}{\vx}}{\Proxy\sach{y}{\vx}}\right)\right]\\
    &=\kl[y\mid\vx]{\ProbitI}{\Proxy} - \kl[y\mid\vx]{\ProbitI}{\ProxyR}.
\end{align*}
On expectation over $\vx\sim\Ideal_\vx$ we obtain:
\begin{equation*}
    \E\left[\left.\frac{\partial\Deltagamma(\vx;\eta)}{\partial \eta}\right|_{\eta=0}\right]=
    \E\kl[y\mid\vx]{\Ideal}{\Proxy} - \E\kl[y\mid\vx]{\Ideal}{\ProxyR}.
\end{equation*}
This quantity is negative if $\Proxy$ is a better proxy of $\Ideal$ than $\ProxyR$.

In the same way, the second derivative verifies, $\forall \vx\in\sX$:
\begin{equation*}
    \frac{\partial^2\Deltagamma(\vx;\eta)}{\partial^2 \eta} =
    \frac{\E_{y\sim\ProbitI\sach{y}{\vx}} \left[A^\eta(y,\vx)\log^2\left(A(y,\vx)\right)\right]}
    {\E_{y\sim\ProbitI\sach{y}{\vx}}\left[A^\eta(y,\vx)\right]}
    - \frac{\E_{y\sim\ProbitI\sach{y}{\vx}} \left[A^\eta(y,\vx)\log\left(A(y,\vx)\right)\right]^2}{\left(\E_{y\sim\ProbitI\sach{y}{\vx}}\left[A^\eta(y,\vx)\right]\right)^2}\geq 0,
\end{equation*}
thanks to the Cauchy-Schwarz inequality.

Finally, $Z$ is coercive for non zero logits shift because, in this case, at least one of the coordinates $y$ verifies $\frac{\Proxy\sachq{r}{y}{\vx}}{\Proxy\sach{y}{\vx}}>1$.

To recap, $\Deltagamma(\eta)\doteq \E\left[\Deltagamma(\vx;\eta)\right]$ enjoys the following properties:
\begin{itemize}
    \item $\Deltagamma(0) = 0$
    \item $\Deltagamma(\eta)$ is convex over $[0,+\infty)$
    \item $\Deltagamma$ is coercive for non zero logits shift
    \item $\Deltagamma^\prime(0) \leq 0$ if $\Proxy$ is a better proxy of $\Ideal$ than $\ProxyR$.
\end{itemize}
Consequently, under this assumption, there exists $\eta_{\max}>0$ such that $\Deltagamma(\eta) \leq 0$ for all $0\leq\eta\leq\eta_{\max}$.

In practice,  $\eta_{\max}$ is found by a line search as the zero of the convex Monte-Carlo approximation:
\begin{equation}
\label{eq:MaxInPratice}
\hat Z(\eta)\doteq \frac{1}{\n}\sum_{\vx\in\train}\lse(\LogitI(\vx)+\eta\,\ProxyShift(\vx))-\lse(\LogitI(\vx)) \simeq \Deltagamma(\eta).
\end{equation}
This is summarized in Alg.~\ref{alg:etamax}.

\paragraph{Remark on the interpretation of $\eta_{\max}$.} The careful reader will spot that the derivative of $\Deltagamma$ cancels out exactly at $\eta_0$: the regime of indistinguishability between $\Proxy$ and $\ProxyR$.
We have $\eta_0 \leq \eta_{\max}$ due to the negative derivative hypothesis.
Taking $\eta=\eta_{\max}$ ensures that we have been through the equality boundary and that we are in the regime where $\E\kl[y\mid\vx]{\ProbitU_\eta}{\Proxy} > \E\kl[y\mid\vx]{\ProbitU_\eta}{\ProxyR}$.

\paragraph{Remark on the algorithmic convention.} Alg~\ref{alg:etamax} needs an upper bound $\eta_{up}$ for initializing the search.
If $\hat Z(\eta_{up})\leq 0$ returning $\eta_{\max}=\eta_{up}$ remains safe.
But the true zero may lie beyond $\eta_{up}$.
Then, the bisection might be preceded by a phase that
geometrically increases the upper bound
until $\hat Z(\eta_{\mathrm{up}})>0$.
This corresponds to the classical exponential search for unbounded
problems~\citep{Bentley1976AnAO}.
Section~\ref{sec:experiments} sets $\eta_{up} = 100$ to avoid this algorithmic complexity.
The algorithmic choice was largely sufficient for our setup (see Fig.~\ref{fig:search-etamax}).

\begin{figure}[bt]
  \centering
  \includegraphics[width=\columnwidth, clip]{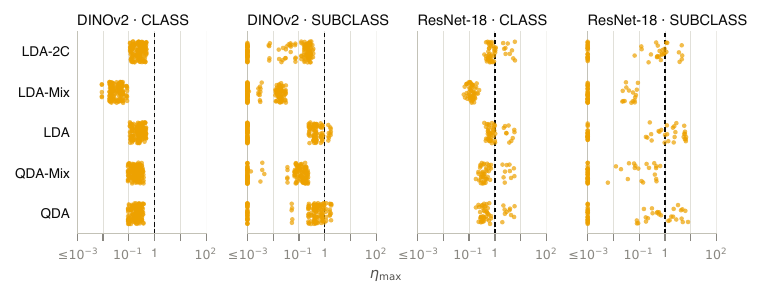}
  \caption{
  Computation of $\eta_{\max}$ from the proxies, over scenarios and architecture.
  One dot corresponds to one run of Alg.~\ref{alg:etamax} across seeds, specific class and subclass scenario.
  Results for the three heads of DINOv2 are aggregated onto one subfigure.
  }
  \label{fig:search-etamax}
\end{figure}

\begin{algorithm}[t]
\caption{\textsc{EtaMax} --- the safety-range line search.}
\label{alg:etamax}
\begin{algorithmic}[1]
\Require net $\LogitI$, log ratios $\nicefrac{\Proxy\sachq{r}{y}{\vx}}{\Proxy\sach{y}{\vx}}$, $\mathcal{D}_{v}$ held-out retain data, upper search $\eta_{\mathrm{up}}$, tolerance $\tau$.
\Ensure Safety range $\eta_{\max}$ of the unlearned family~\eqref{eq:delta-M} with $Z(\eta_{\max})\leq 0$.
\Function{$\hat Z$}{$\eta$} \Comment{the empirical $\Deltagamma$ of~\eqref{eq:MaxInPratice}}
    \State \Return $\frac{1}{\n}\sum_{\vx\in\train}\bigl[\lse\bigl(\LogitI(\vx)+\eta\,\ProxyShift(\vx)\bigr)-\lse\bigl(\LogitI(\vx)\bigr)\bigr]$
\EndFunction
\State $\hat Z ^\prime(0) \gets \frac{1}{\n}\sum_{\vx\in\train}\sum_{y=1}^{C}\ProbitI(\vx)_y\log\bigl(\nicefrac{\Proxy\sachq{r}{y}{\vx}}{\Proxy\sach{y}{\vx}}\bigr)$
\State $\hat{\mathcal{S}}_r \gets -\sum_{(\vx,y)\in\mathcal{D}_{v}}\log\bigl(\nicefrac{\Proxy\sachq{r}{y}{\vx}}{\Proxy\sach{y}{\vx}}\bigr)$ \Comment{empirical estimate of Prop.~\ref{prop:decrease} assumption (see App.~\ref{app:check_prop1})}
\If{$\hat Z ^\prime(0) > 0 \textbf{ or } \hat{\mathcal{S}}_r > 0$}
    \State \Return $\eta_{\max} = 0$ \Comment{safe if not decreasing or bad retain modeling}
\EndIf
\State $\eta_{\mathrm{down}} \gets 0$
\If{$\hat Z(\eta_{\mathrm{up}}) \leq 0$}
    \State \Return $\eta_{\mathrm{up}}$
\EndIf
\While{$\eta_{\mathrm{up}}-\eta_{\mathrm{down}} > \tau$}\Comment{invariant: $\hat Z(\eta_{\mathrm{down}})\leq 0 < \hat Z(\eta_{\mathrm{up}})$}
    \State $\eta \gets \nicefrac{(\eta_{\mathrm{down}}+\eta_{\mathrm{up}})}{2}$
    \State \textbf{if} $\hat Z(\eta)\leq 0$ \textbf{then} $\eta_{\mathrm{down}}\gets\eta$ \textbf{else} $\eta_{\mathrm{up}}\gets\eta$ \Comment{convexity makes the zero unique}
\EndWhile
\State \Return $\eta_{\max} = \eta_{\mathrm{down}}$ \Comment{conservative end so that $\hat Z(\eta_{\max})\leq 0$ holds}
\end{algorithmic}\par
\end{algorithm}

\section{Main results}
\label{app:uphlimitations}

\subsection{Study of the logits error and proof of Prop.~\ref{prop:decrease}}
\label{app:Decrease}
We now define the error of prediction between the true distribution $\IdealR\sach{y}{\vx}$ and the unlearned model $\ProbitU_\eta$, in log scale:
\begin{align*}
    \gE(y,\vx;\eta) &\doteq \log(\Ideal\sachq{r}{y}{\vx})-\log(\ProbitU\sachq{\eta}{y}{\vx}).
\end{align*}
This quantity is worth studying as it appears in the computation of $\E\kl[y\mid\vx]{\IdealR}{\ProbitU_\eta}$, the quantity of importance in this paper. We have:
\begin{align}
    \kl[y\mid\vx]{\IdealR}{\ProbitU_\eta}
    &= \sum_{y=1}^C \Ideal\sachq{r}{y}{\vx}\log\left(\frac{\Ideal\sachq{r}{y}{\vx}}{\ProbitU\sachq{\eta}{y}{\vx}}\right)
    =\langle \Ideal\sachq{r}{\cdot}{\vx}, \gE(\cdot,\vx,\eta)\rangle.
    \label{lem:linkKL}
\end{align}

On the other hand, we have
\begin{align}
    \gE(\vx,y;\eta)
    &= \log(\Ideal\sachq{r}{y}{\vx}) -\LogitU_\eta(\vx)+\gammaU(\vx;\eta)\nonumber\\
    &\stackrel{\eqref{eq:delta-M}}{=} \log(\Ideal\sachq{r}{y}{\vx}) -\LogitI(\vx) - \eta\ProxyShift(\vx)+\gammaU(\vx;\eta)\nonumber\\
    &= \log(\Ideal\sachq{r}{y}{\vx}) -\log\ProbitI\sach{y}{\vx} -\gamma(\vx) - \eta\ProxyShift(\vx)+\gammaU(\vx;\eta)\nonumber\\
    &= \log\left(\frac{\Ideal\sachq{r}{y}{\vx}}{\ProbitI\sach{y}{\vx}}\right) - \eta \log\left(\frac{\Proxy\sachq{r}{y}{\vx}}{\Proxy\sach{y}{\vx}}\right) + \Deltagamma(\vx;\eta)\label{eq:decrease}
\end{align}

Plugging~\eqref{eq:decrease} into~\eqref{lem:linkKL} and taking the expectation leads to:
\begin{align*}
        \E\kl[y\mid\vx]{\IdealR}{\ProbitU_\eta} &= \E\kl[y\mid\vx]{\IdealR}{\ProbitI} \\
        &+ \underbrace{\E\left[\Deltagamma(\vx,\eta)\right]}_{\leq 0, \forall \eta\in[0,\eta_{\max}]}
        + \eta \left(\underbrace{\E\kl[y\mid\vx]{\IdealR}{\ProxyR} - \E\kl[y\mid\vx]{\IdealR}{\Proxy}}_{\leq 0\text{ if $\ProxyR$ is a better proxy for $\IdealR$ than $\Proxy$ is}}\right).\nonumber
\end{align*}
The first term is the KL divergence without unlearning, \ie\ between the true retain distribution and the initial NN probits.
App.~\ref{par:normalization} shows that the second term is negative for any $0\leq\eta\leq\eta_{\max}$.
The third term is negative with one extra assumption: if
$\ProxyR$ is a better proxy for $\IdealR$ than $\Proxy$ is.

This proves Prop.~\ref{prop:decrease}.

\subsection{Proof of Prop.~\ref{prop:tight-bound}}
\label{App:Prop2}
We rewrite~\eqref{eq:decrease} in order to obtain another expression of the error $\gE$:
\begin{align*}
    \gE(\vx,y;\eta)
    = (1-\eta)\log\left(\frac{\Ideal\sachq{r}{y}{\vx}}{\ProbitI\sach{y}{\vx}}\right)
    +\eta\left(\log\left(\frac{\Ideal\sachq{r}{y}{\vx}}{\Proxy\sachq{r}{y}{\vx}}\right) - \log\left(\frac{\ProbitI\sach{y}{\vx}}{\Proxy\sach{y}{\vx}}\right)\right) + \Deltagamma(\vx;\eta).
\end{align*}
We now take the expectation w.r.t. $\IdealR\sach{y}{\vx}$ and then the expectation w.r.t. $\Ideal_\vx$.
For $\eta\in[0,\eta_{\max}]$, $\Deltagamma(\eta)$ is non-positive. This leads to:
\begin{equation*}
        \E\kl[y\mid\vx]{\IdealR}{\ProbitU_\eta} \leq (1-\eta)\E\kl[y\mid\vx]{\IdealR}{\ProbitI} + \eta \left( \E\kl[y\mid\vx]{\IdealR}{\ProxyR}  - \underbrace{\E\sum_{y=1}^C\Ideal\sachq{r}{y}{\vx} \log\frac{\ProbitI\sach{y}{\vx}}{\Proxy\sach{y}{\vx}}}_{\text{term $A$}}\right)
\end{equation*}
Note that we get an equality when $\eta=\eta_{\max}$ because $\Deltagamma(\eta_{\max})=0$ by definition of $\eta_{\max}$.
The following lemma helps us work out term $A$.

\begin{lemma}[Formal Calculus]
According to the mixture~\eqref{eq:Amix} of the true distributions, $\Ideal = (1-\pi_f)\IdealR + \pi_f \IdealF$ over the joint data $(\vx,y)$.
Then we have the following formal identity with respect to $\vx$, $y$, or $(\vx,y)$:
\label{lem:formal}
    \begin{equation}
        \E = (1-\pi_f)\E_{r} + \pi_f\E_{f}
          = \E_{r} + \pi_f\left(\E_{f} - \E_{r} \right) = \E_{r} + \gO{(\pi_f)}.
        \label{eq:AppFormal}
    \end{equation}
\end{lemma}

Applying this lemma to the term
$A=\E\left[\rz\right]$, where we abbreviate
$\rz\doteq\sum_{y=1}^C\Ideal\sachq{r}{y}{\vx}\log\frac{\ProbitI\sach{y}{\vx}}{\Proxy\sach{y}{\vx}}$:
\begin{align*}
    A&\stackrel{\eqref{eq:AppFormal}}{=}\E_{r}[\rz]
    + \pi_f\left(\E_{f}\left[\rz\right]
    - \E_{r}[\rz] \right).
\end{align*}

\begin{lemma}[Expectation Telescoping]
\label{lem:telescope}
For a measure $\mu$ over $(\vx,y)$ and any test function $\phi(\vx,y)$ we have the following identity:
    \begin{equation}
    \label{eq:telescope}
        \E_{\mu}\left[\phi(\vx,y)\right] = \E_{\mu}\left[\sum_{y=1}^C \mu(y\mid\vx) \phi(\vx,y)\right].
    \end{equation}
On the right-hand side of the equation, the expectation is taken on a quantity that do not depend on $y$.
\end{lemma}

We note $\Lambda(\vx,y)\doteq\log\frac{\ProbitI\sach{y}{\vx}}{\Proxy\sach{y}{\vx}}$ and use lemma~\ref{lem:telescope} to write $\E_{\vx\sim\IdealR}[\rz]=\E_{(\vx,y)\sim\IdealR}[\Lambda(\vx,y)]$.

Expanding $\E_{r}[\rz]$ with the same lemma, now applied on the joint $(\vx,y)$,
\begin{align*}
    A
    &\stackrel{\eqref{eq:AppFormal}}{=}\E_{r}[\rz]
    + \pi_f\left(\E_{f}\left[\rz\right]
    - \E_{r}[\rz] \right)\\
    &\stackrel{\eqref{eq:telescope}}{=} \E_{r}[\Lambda(\vx,y)]
    + \pi_f\left(\E_{f}\left[\rz\right]
    - \E_{r}[\rz] \right)\\
    &\stackrel{\eqref{eq:AppFormal}}{=} \E[\Lambda(\vx,y)] - \pi_f (\E_{f}[\Lambda(\vx,y)] - \E_{r}[\Lambda(\vx,y)])
    + \pi_f\left(\E_{f}\left[\rz\right]
    - \E_{r}[\rz] \right)\\
    &\stackrel{\eqref{eq:telescope}}{=} \E[\Lambda(\vx,y)] - \pi_f \E_{f}[\Lambda(\vx,y)]
    + \pi_f \E_{f}\left[\rz\right]
     \\
    &=\underbrace{\E\left[\Lambda(\vx,y)\right]}_{\text{term $A_1$}}
    +\underbrace{\pi_f\E_{f}\left[\sum_{y=1}^C\left( (\Ideal\sachq{r}{y}{\vx}- \Ideal\sachq{f}{y}{\vx})\Lambda(\vx,y) \right)\right]}_{\text{term $A_2$}},
\end{align*}
where the last line uses $\E_{f}[\Lambda(\vx,y)]\stackrel{\eqref{eq:telescope}}{=}\E_{f}\left[\sum_{y}\Ideal\sachq{f}{y}{\vx}\Lambda(\vx,y)\right]$.

The first term $A_1$ is $\E\left[\Lambda(\vx,y)\right] = \E\left[\log\frac{\ProbitI\sach{y}{\vx}}{\Proxy\sach{y}{\vx}}\right]$. Here, we separate the unlearning error from the learning error:
\begin{equation}
\label{eq:learning-error}
    \E\kl[y\mid\vx]{\Ideal}{\ProbitI}= \E\left[\log\frac{\Ideal\sach{y}{\vx}}{\ProbitI\sach{y}{\vx}}\right]
\end{equation}
so that we can state the bound in terms of divergence using~\eqref{eq:telescope} with
\begin{equation}
    A_1 = \E\kl[y\mid\vx]{\Ideal}{\Proxy} - \E\kl[y\mid\vx]{\Ideal}{\ProbitI}.
\end{equation}

The second term $A_2$ is not really tractable and will be referred to as a residual $\gO{(\pi_f)}$:
\begin{equation}
    A_2 = \pi_f\E_{f}\left[\sum_{y=1}^C\left( (\Ideal\sachq{r}{y}{\vx}- \Ideal\sachq{f}{y}{\vx})\log\frac{\ProbitI\sach{y}{\vx}}{\Proxy\sach{y}{\vx}} \right)\right] = \gO{(\pi_f)}.
    \label{eq:Opif}
\end{equation}
The residual is close to zero either when $\Ideal\sachq{f}{y}{\vx}$ and $\Ideal\sachq{r}{y}{\vx}$ take similar values on most likely forget data points (no unlearning signal), or when the initial proxy is close to the initial distribution.

In the end, since $A$ enters the bound with a minus sign and a factor $\eta$, we obtain $\E\kl[y\mid\vx]{\IdealR}{\ProbitU_\eta}$ lower or equal to
\begin{equation}
         (1-\eta)\E\kl[y\mid\vx]{\IdealR}{\ProbitI} + \eta \left( \E\kl[y\mid\vx]{\IdealR}{\ProxyR}  - \E\kl[y\mid\vx]{\Ideal}{\Proxy}\right) -\eta\, \gO{(\pi_f)} + \eta \E\kl[y\mid\vx]{\Ideal}{\ProbitI}.
        \label{eq:AppTight}
\end{equation}
This proves Prop.~\ref{prop:tight-bound}.

\subsection{Proof of Prop.~\ref{thm:mainTHM}}
\label{app:Prop3}
This section converts the previous bound with probabilities $y\mid\vx$ into terms modeling $\vx\mid y$ thanks to the following lemma.
\begin{lemma}
\label{lem:AppXY}
    For any distributions $P$ and $Q$ over $(\vx,y)$ with the same marginal $P(y)=Q(y)$:
    \begin{align}
        \E\left[\kl[y\mid\vx]{P}{Q}\right]
        &= \E\left[\kl[\vx\mid y]{P}{Q}\right] - \kl{P_\vx}{Q_\vx}
        \label{eq:AppXY}
    \end{align}
\end{lemma}
\begin{proof}
\begin{align*}
        \E\left[\kl[y\mid\vx]{P}{Q}\right]
        &=\E_{\vx\sim P}\left[\kl[y\mid\vx]{P}{Q}\right]\\
        &= \E_{\vx\sim P}\left[\E_{y\sim P\sach{y}{\vx}}\left[\log\frac{P\sach{y}{\vx}}{Q\sach{y}{\vx}}\right]\right]\\
        &=\E_{(\vx,y)\sim P}\left[\log\frac{P\sach{y}{\vx}}{Q\sach{y}{\vx}}\right]\\
        &\stackrel{\text{Bayes}}{=}\E_{(\vx,y)\sim P}\left[\log\frac{P\sach{\vx}{y}Q(\vx)}{Q\sach{\vx}{y}P(\vx)}\right]\\
        &= \E_{y\sim P}\left[\kl[\vx\mid y]{P}{Q}\right] - \kl{P_\vx}{Q_\vx}\\
        &= \E\left[\kl[\vx\mid y]{P}{Q}\right] - \kl{P_\vx}{Q_\vx},
    \end{align*}
    because $\E_{(\vx,y)\sim P}\left[\log\frac{P(\vx)}{Q(\vx)}\right] = \E_{\vx\sim P}\left[\log\frac{P(\vx)}{Q(\vx)}\right]=\kl{P_\vx}{Q_\vx}$ considers only the marginal distributions of $\vx$.
\end{proof}

By definition of our proxies they have the same marginal in $y$ than the distribution they are attributed to.
Lemma~\ref{lem:AppXY} is legitimate on both lines: $\Proxy_{r,y}=\Ideal_{r,y}$ and
$\Proxy_y=\Ideal_y$ hold by construction of the proxy posteriors~\eqref{eq:posterior},
which reuse the true class priors. Applying Lemma~\ref{lem:AppXY} to the second term of the previous bound leads to:
\begin{align*}
        \E\kl[y\mid\vx]{\IdealR}{\ProxyR}  - \E\kl[y\mid\vx]{\Ideal}{\Proxy}
        &\stackrel{\eqref{eq:AppFormal}}{=} \E_{r}\kl[y\mid\vx]{\IdealR}{\ProxyR} - \E\kl[y\mid\vx]{\Ideal}{\Proxy} + \gO{(\pi_f)}\\
        &\stackrel{\eqref{eq:AppXY}}{=} \E_{r}\kl[\vx\mid y]{\IdealR}{\ProxyR}-\kl{\Ideal_{r,\vx}}{\Proxy_{r,\vx}}\\
        &\quad- \E\kl[\vx\mid y]{\Ideal}{\Proxy} +\kl{\Ideal_\vx}{\Proxy_\vx} + \gO{(\pi_f)}.
\end{align*}

This proves Prop.~\ref{thm:mainTHM}.

\section{Empirical Check of the Condition of Prop.~\ref{prop:decrease}}
\label{app:check_prop1}

\subsection{Proposed test}

Proposition~\ref{prop:decrease} guarantees a decrease of the expected KL divergence if
$$\mathcal{S}\doteq\E_{\vx\sim\Ideal_\vx}\bigl[\kl[y\mid\vx]{\IdealR}{\ProxyR}-\kl[y\mid\vx]{\IdealR}{\Proxy}\bigr]\le 0.$$
By the mixture of distributions~\eqref{eq:Amix}, we rewrite $\mathcal{S}=\pi_r\mathcal{S}_r+\pi_f\mathcal{S}_f$, where $\mathcal{S}_r$
and $\mathcal{S}_f$ are the same expectation but over the distributions $p_r(\vx)$ and $p_f(\vx)$ respectively.
The term $\mathcal{S}_r$ simplifies into:
$$\mathcal{S}_r=\E_{(\vx,y)\sim\IdealR}\bigl[\log\Proxy\sach{y}{\vx}-\log\ProxyR\sach{y}{\vx}\bigr].$$
This quantity is empirically measurable by the following estimator:
\begin{equation}
  \hat{\mathcal{S}}_r\doteq\frac{1}{|\train_v|}\sum_{(\vx,y)\in\train_v}\log\frac{\Proxy\sach{y}{\vx}}{\ProxyR\sach{y}{\vx}},
  \label{eq:Shat_r}
\end{equation}
where $\train_v$ is a validation held-out set of retained data. In other words, the identification of the parameters of the proxies did not use this dataset.

Unfortunately, this simplification does not occur for the second term $\mathcal{S}_f$. It remains non-observable because the true distribution $\IdealR\sach{y}{\vx}$ is unknown on the forget set. However, its impact on $\mathcal{S}$ is light as $\pi_f$ is usually very small.

In the end, the unlearner might estimate the sign of $\mathcal{S}$ without retraining via the empirical
average $\hat{\mathcal{S}}_r$. Algorithm~\ref{alg:etamax} includes this test: If $\hat{\mathcal{S}}_r>0$,
the algorithm returns $\eta_{\max}=0$.

\subsection{Empirical validation}
\paragraph{Protocol.}
We evaluate this test in the \subclass{} scenario with the naive \LDAI{} proxy
($\ProxyR$ fitted on $\train_r$, $\Proxy$ on $\train$) on CIFAR-10 and DINOv2+mlp1 architecture.
We compute the estimator $\hat{\mathcal{S}}_r$ on held-out data never seen by the models.
We cannot evaluate the true quantity $\mathcal{S}$ because the true distribution $\IdealR$ is unknown.
In the experiments, we \emph{estimate} $\mathcal{S}$ by using an ideal retrained model $\ProbitR$, unavailable in practice.

Besides the normal runs, we artificially degrade the proxies by two different means.
First, we fit $\ProxyR$ on a small fraction of the retained set $\train_r$.
Second, in the degraded features setup, we fit the proxies on less informative features corrupted by noise $\sigma$:
$\phi(\vx)+\sigma\bm{\varepsilon}$ with $\bm{\varepsilon}\sim\gN(\bm{0},\mI)$.

\begin{figure}
    \centering
    \includegraphics[width=\linewidth]{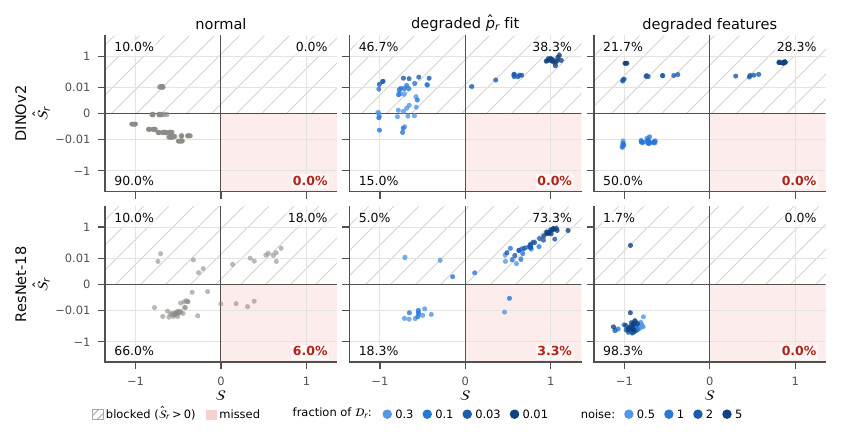}
  \captionof{figure}{\subclass{} scenario. Observable $\hat{\mathcal{S}}_r$ vs.\
  true $\mathcal{S}$ across normal and degraded setups (one point per run, \% of runs
  per quadrant). Hatched area ($\hat{\mathcal{S}}_r>0$) shows runs blocked by the test ($\eta_{\max}=0$).
  Red area ($\mathcal{S}>0$ and $\hat{\mathcal{S}}_r\le 0$) shows undetected runs where $\ProxyR$ is less accurate than $\Proxy$.}
  \label{fig:check_prop1}
\end{figure}

\paragraph{Results.}

First, it is not true that the condition always holds true. Figure~\ref{fig:check_prop1} shows that a substantial number of runs provides $\mathcal{S}>0$ when the proxies or the features are degraded, artificially (corrupted by noise) or not (ResNet-18).

Our test helps filtering out these cases. It makes none or few undetected violations (\ie, points in the red quadrant in Fig.~\ref{fig:check_prop1}).
The number of unnecessary abstentions is larger but less impactful: $\eta_{\max}$ is set to zero and no unlearning happens.
Surprisingly, noise helps on ResNet-18.
Every run with degraded features verifies $\mathcal{S}\le 0$.
The noise seems to improve the modeling by flattening $\ProxyR$.
Our intuition is that the clean proxy is overconfident on held-out data, and a flattened distribution is indeed closer to $\IdealR$.

\section{About discriminant analysis proxies}
\label{app:homoscedastic}

In the experiments Sect.~\ref{sec:experiments} the reported $\qda$ are taken with diagonal covariances for sake of computational cost. It explains why they can under-perform our $\lda$ approaches. This section derives the cost of not taking the complete covariances in our $\qda$ experiments, linking it directly to the performances of our $\lda$ methods (see Prop.~\ref{prop:homo-hypo}).

\paragraph{Homoscedastic Gaussian data.}
Linear Discriminant Analysis ($\lda$) assumes a generative model where each class follows a Gaussian distribution with a shared covariance:
\begin{equation}
    \vx \mid y=k \;\sim\; \gN(\vmu_k,\, \mSigma).
    \label{eq:lda_model}
\end{equation}
The posterior probability has the following expression~\citep{Mclachlan2004discriminant}:
\[
\Ideal\sach{y=k}{\vx}
= \frac{\exp(\vw_k^\top \vx + b_k)}{\sum_{j=1}^C \exp(\vw_j^\top \vx + b_j)},
\]
with row weights $\vw_k = \mSigma^{-1}\vmu_k$ and biases $b_k = \ln\pi_k - \tfrac{1}{2}\vw_k^\top\vmu_k$, where $\pi_k$ is the class prior.
The empirical moment $(\hat\pi_k, \hat{\vmu}_k, \hat{\mSigma})$ estimated from $\train$ yield the Maximum Likelihood Estimator which exactly has the form of a linear layer $\left(\mW_\lda=(\vw_1,\ldots,\vw_C)^\top, \vb_\lda=(b_1,\ldots,b_C)^\top\right)$ followed by the softmax operator.
Indeed, as $n\to\infty$, this MLE coincides with the global minimizer of the cross-entropy loss over any linear model, making $\lda$ the Bayes-optimal classifier in this setting~\citep{Mclachlan2004discriminant,Papyan2020collapse}. This explains the success of last-layer linear probing on strongly pre-trained features~\citep{Kumar2022lpft,Alain2016probes}: when frozen features are approximately class-conditional Gaussian distributed, fitting a linear head by cross-entropy is equivalent to computing the MLE in closed form.

\paragraph{Heteroscedastic Gaussian data.}
When the homoscedasticity assumption is relaxed, each class $k$ has its own covariance $\mSigma_k$:
\begin{equation}
    \vx \mid y=k \;\sim\; \gN(\vmu_k,\, \mSigma_k).
    \label{eq:qda_model}
\end{equation}
This is Quadratic Discriminant Analysis ($\qda$). Define the quadratic discriminant as
\begin{equation}
 \sq_k(\vx) = -\tfrac{1}{2}(\vx-\vmu_k)^\top \mSigma_k^{-1} (\vx-\vmu_k)
                   - \tfrac{1}{2}\ln|\mSigma_k| + \ln\pi_k,
    \label{eq:qda_discriminant}
\end{equation}
so that the posterior probability $\Ideal\sach{y=k}{\vx} = \softmax\left(\sq(\vx)\right)_k$~\citep{Mclachlan2004discriminant}.
Under the heteroscedastic Gaussian assumption and in the limit $n\to\infty$, minimizing the cross-entropy loss drives the neural network's predictive distribution to this softmax~\citep{Mclachlan2004discriminant}.

Even when the Gaussian assumption is only approximate, $\lda$ and $\qda$ remain powerful and  computationally efficient proxies as
confirmed by the strong linear-probing results of recent foundation models~\citep{oquab2023dinov2,Kumar2022lpft}.
Figure~\ref{fig:mon_image} illustrates this equivalence on a toy example.

\begin{figure}[btp]
    \centering
    \includegraphics[width=\columnwidth]{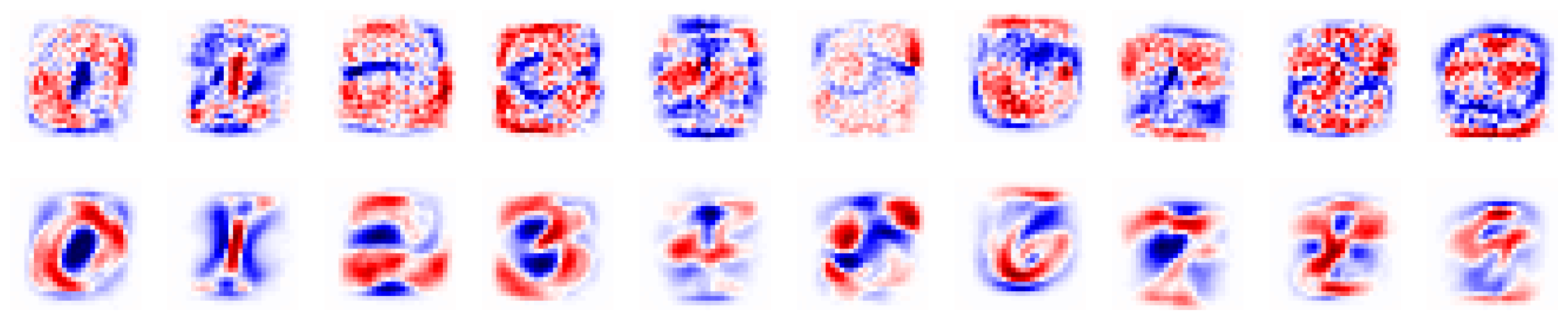}
    \caption{
      \textbf{$\lda$ matches cross-entropy training weights.}
      Top: last-layer weight matrix (rows reshaped as images) of a linear network trained with the Adam optimizer and cross-entropy loss.
      Bottom: closed-form $\lda$ weight formula.
      Dataset: MNIST~\citep{LeCun1998MNIST}.
      The near-perfect visual match empirically confirms the cross-entropy/$\lda$ equivalence.
    }
    \label{fig:mon_image}
\end{figure}

\begin{figure}[htbp]
    \centering
    \includegraphics[width=\textwidth]{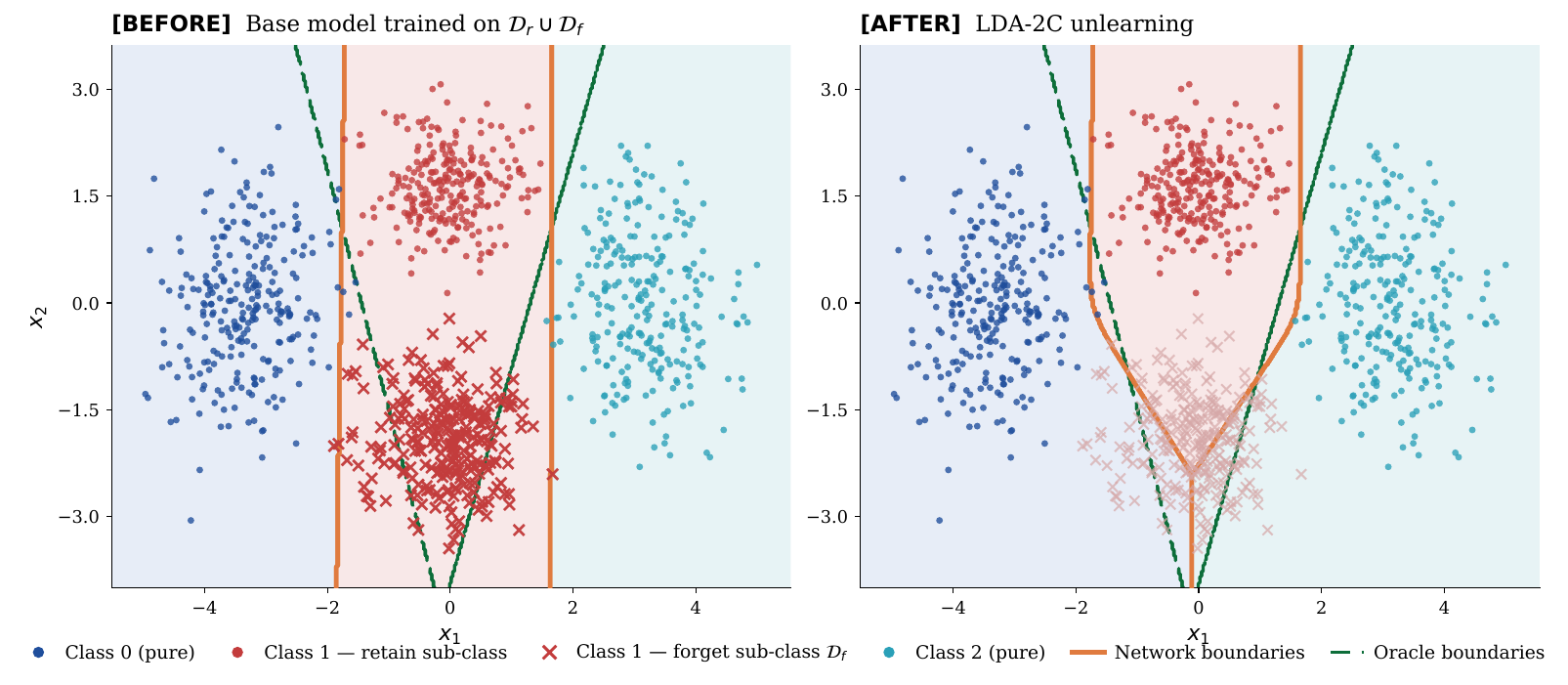}
        \caption{Illustration of the sub-class scenario over synthetic heteroscedastic Gaussian data. Class 1 is composed of two sub-classes, green and red points. The latter have to be unlearned. }
    \label{fig:illustration}
\end{figure}

\paragraph{Cost of Homoscedastic Assumption.}
This subsection exhibits the cost to model the retain proxy by $\lda$ rather than $\qda$.

\begin{theorem}[Pythagorean Theorem of Information~\cite{csiszar1975divergence}.]
Let $\sP$ be a closed convex set of probability distributions,
$p_0 \notin \sP$, and $p^\star = \argmin_{p \in \sP} \kl{p}{p_0}$
the I-projection of $p_0$ onto $\sP$.
Then for any $p \in \sP$:
\begin{equation}
    \kl{p}{p_0} \geq \kl{p}{p^\star} + \kl{p^\star}{p_0},
    \label{eq:pythagoras-general}
\end{equation}
with equality when $\sP$ is a linear family.
\end{theorem}

We apply this theorem class by class conditioning on $y=k$ for $k \in \{1,\ldots,C\}$.
Let us define $\sP_k$ as the family of probabilities with mean $\vmu_k$ and covariance $\mSigma_k$:
\begin{equation}
    \sP_k \doteq \left\{P:\, \E_P[\vx] = \vmu_k, \E_P[(\vx-\vmu_k)(\vx-\vmu_k)^T] = \mSigma_k\right\}.
\end{equation}
This is a closed convex set of probabilities and a linear family
in the sense of~\eqref{eq:pythagoras-general}, as it is defined
by linear equality constraints on the sufficient statistics
($P \mapsto \E_P[\vx]$ and $P \mapsto \E_P[\vx\vx^\top]$
are linear functionals of $P$),
so the Pythagorean equality holds exactly.

The I-projection of~\cite{https://doi.org/10.1002/ecja.4400660602} of $p_0 = \Proxy^{\lda}_{r}\sach{\vx}{y=k} \notin \sP_k$
onto $\sP_k$ is indeed $\Proxy^{\qda}_{r}\sach{\vx}{y=k}$, $\ie$ $\gN(\vmu_k,\mSigma_k)$.
We write $\kl{p}{p_0}=\E_p[-\log p_0(\vx)]-H(p)$. Since $p_0$ is a Gaussian distribution, $-\log p_0(\vx)$ is quadratic in $\vx$ whose expectation depends only on the statistics $(\vmu_k,\mSigma_k)$.
In other words, this cross-entropy term is constant over $\sP_k$.
Minimizing $\kl{p}{p_0}$ then amounts to maximzing the entropy $H(p)$ over $\sP_k$, which is achieved by $\gN(\vmu_k,\mSigma_k)$~\cite{cover2006elements}.
\begin{equation*}
    p^\star
    = \argmin_{p \in \sP_k} \kl{p}{\Proxy^{\lda}_{r}\sach{\vx}{y=k}}= \Proxy^{\qda}_{r}\sach{\vx}{y=k}.
\end{equation*}

Assuming that $\IdealR\sach{\vx}{y=k}\in\sP_k$, the equality holds exactly:
\begin{equation*}
    \kl[\vx\mid y=k]{\IdealR}{\Proxy^\lda}
    = \kl[\vx\mid y=k]{\IdealR}{\Proxy^\qda}
    + \kl[\vx\mid y=k]{\Proxy^\qda}{\Proxy^\lda}.
\end{equation*}

\begin{lemma}
The KL between two Gaussians with the same mean $\vmu_k$ and different
covariances $\mSigma_k$ and $\mSigma$ respectively:
\begin{align*}
    \kl{\gN(\vmu_k,\mSigma_k)}{\gN(\vmu_k,\mSigma)}
    &= \frac{1}{2}\left(
        \log \frac{|\mSigma|}{|\mSigma_k|}
        + \Tr(\mSigma^{-1}\mSigma_k) - d
    \right).
\end{align*}
\end{lemma}

Finally, we link the above lemma to Prop.~\ref{thm:mainTHM} with the following result:
\begin{proposition}
\label{prop:homo-hypo}
The explicit cost of the homoscedastic assumption, \ie\ $\lda$ proxy, on the bound of Prop.~\ref{thm:mainTHM}.
    \begin{equation*}
    \E_{y\sim\IdealR}\kl[\vx\mid y]{\IdealR}{\ProxyR^\lda}
    = \E_{y\sim\IdealR}\kl[\vx\mid y]{\IdealR}{\ProxyR^\qda}
    + \frac{1}{2}\sum_{k=1}^C \IdealR(k)\left(
        \log \frac{|\mSigma|}{|\mSigma_k|}
        + \Tr(\mSigma^{-1}\mSigma_k) - d
    \right).
\end{equation*}
\end{proposition}
An unlearner knows exactly the cost to pay in term of divergence when reducing the computational cost of various inverse covariances.

\section{Exact Unlearning Signal with empirical distribution}
\label{app:exact_dirac}

What happens if we model the data structure with indicator function over the labeled data $(\vx,y(\vx))$ ?

We derive a closed-form solution and discuss its consequences for the unlearned model
and for the distillation~\eqref{eq:lda2c-lstsq-grad}.
We use the framework of mixture models described in Sect.~\ref{sec:our_method}; in particular the mixture
$\Proxy\sach{\vx}{y} = \pi_r(y)\Proxy\sachq{r}{\vx}{y} + \pi_f(y)\Proxy\sachq{f}{\vx}{y}$
and posteriors~\eqref{eq:posterior}.

The unlearner chooses the empirical measures for both proxies:
\begin{equation*}
    \Proxy\sachq{r}{\vx}{y}=\frac{1}{|\train_r(y)|}\1\{(\vx,y)\in\train_r\},
\end{equation*}
and likewise for $\ProxyF$ on $\train_f$. Proxy $\Proxy$ becomes a mixture of empirical measures on $\train$:
\begin{equation*}
    \Proxy\sach{\vx}{y} = \frac{\pi_r(y)}{|\train_r(y)|}\1\{(\vx,y)\in\train_r\} + \frac{\pi_f(y)}{|\train_f(y)|}\1\{(\vx,y)\in\train_f\}.
\end{equation*}
Since $\frac{\pi_r(y)}{|\train_r(y)|} = \frac{|\train_r(y)|}{|\train(y)||\train_r(y)|}= \frac{1}{|\train(y)|}$, we simply recover:
\begin{equation*}
    \Proxy\sach{\vx}{y} = \frac{1}{|\train(y)|}\1\{(\vx,y)\in\train\},
\end{equation*}
which is the initial empirical measure. This is the same as doing \emph{no} mixture.

For any $(\vx,y)\in\train_r$, we have $\ProxyR(y) = \IdealR(y)$ and $\ProxyR(\vx) = \sum_c \IdealR (c)\Proxy\sachq{r}{\vx}{c} = \IdealR(y)\Proxy\sachq{r}{\vx}{y}$.
The Bayes inversion~\eqref{eq:posterior} leaves $\Proxy\sachq{r}{y}{\vx} = \Proxy\sachq{r}{\vx}{y} \IdealR(y) / \ProxyR(\vx) = 1$. This also holds for $\Proxy$.
\begin{equation*}
    \Proxy\sachq{r}{y}{\vx} = \1\{(\vx,y)\in\train_r\},\quad \Proxy\sach{y}{\vx} = \1\{(\vx,y)\in\train\}.
\end{equation*}
This yields the following unlearning signal:
\begin{equation*}
    \ProxyShift(\vx)_y = \log\,\frac{\1\bigl\{(\vx,y)\in\train_r\bigr\}}{\1\bigl\{(\vx,y)\in\train_r\bigr\} + \1\bigl\{(\vx,y)\in\train_f\bigr\}}.
\end{equation*}
Since $\train_r\cap\train_f = \emptyset$, we obtain the closed form
\begin{equation}
    \ProxyShift(\vx)_y = \log \1\bigl\{(\vx,y)\in\train_r\bigr\}.
    \label{eq:dd_DeltaM}
\end{equation}
The unlearn model $\LogitU$ completely annihilates the $(\vx,y)$ data (logit set to $-\infty$) unless it belongs to the retain set where the logit is left unchanged.

\paragraph{No more dependency on $\eta$.}
For any $\eta>0$, the $-\infty$ entries dominates the softmax of
$\LogitU_\eta$ uniformly, leaving the truncated probit in the limit:
\begin{equation}
\forall \vx\notin\train_r,\,
    {\ProbitU_\eta(\vx)}_c =
    \begin{cases}
        0 & c = y(\vx), \\
        \ProbitI(\vx)_c\,/\,(1-\ProbitI(\vx)_{y(\vx)}) & c \neq y(\vx).
    \end{cases}
    \label{eq:dd_target_softmax2}
\end{equation}
The truncated probit~\eqref{eq:dd_target_softmax2} carries zero mass at
$y$ on the forget set.
On the other hand, for all $\vx\in\train_r$, $\ProbitU_\eta(\vx) = \ProbitI(\vx)$.
The unlearned probit is $\eta$-independent.

\paragraph{Consequence on the distillation pipeline}
As it is, this correction makes little sense in practice, \ie\ for unseen testing data.
Yet, the KL distillation~\eqref{eq:lda2c-lstsq-grad} over $\train$ remains well-posed because
$\kl{\ProbitU}{\softmax(\bm{\ell}_{\vtheta+\delta\vtheta})}$ uses $\ProbitU$ as the
left argument and $0\cdot\log 0 = 0$. The KL distillation drives the unlearned NN to get closer to the
truncated probit~\eqref{eq:dd_target_softmax2}, but \emph{not} exactly equal to.
This makes a strong decrease of the probit related to $y(\vx)$ for the forget data.

\paragraph{Empirical exact unlearning in $2C$ dimensions}
\label{app:exact_dirac_2c}
The previous truncation~\eqref{eq:dd_target_softmax2} is too aggressive, especially in the \subclass\ scenario.
We lift this truncation from
labels $y\in\llbracket C\rrbracket$ to 2C-dimensions labels $(y, s) \in \llbracket C\rrbracket\times\{r,f\}$.
This amounts to canceling the mass probability of $(y(\vx), f)$ for $\vx\in\train_f$ and redistributing this quantity to the remaining cells $(y,s)$.
Each $s$-row needs a renormalization constant, equal to
the probability $\pi_s(\vx)$ of a given sample $\vx$ to be in the retain ($s=r$) or the forget set ($s=f$).
The simplest inference uses a feature agnostic assumption: $\Proxy\sach{s}{\vx,y} \simeq \Proxy\sach{s}{y} =\pi_s(y)=\nicefrac{|\train_s(y)|}{|\train(y)|}$:
\begin{equation*}
    \pi_s(\vx) = \sum_y \Proxy\sach{s}{\vx,y} \Proxy\sach{y}{\vx}
    \simeq \hat\pi_s(\vx) \doteq \sum_y \Proxy\sach{y}{\vx} \pi_s(y).
\end{equation*}
Since $\Proxy\sach{y}{\vx} = \1\bigl\{(\vx,y)\in\train\bigr\}=\1\{y=y(\vx)\}$, then $\hat\pi_s(\vx) = \pi_s (y(\vx))$.
One can retrieve the dimension $C$ classifier:
\begin{equation*}
    \ProbitU(\vx)
    \doteq
    \hat\pi_r(\vx)\;\ProbitI(\vx) + \hat\pi_f(\vx)\;\ProbitU_f(\vx),
    \label{eq:dd2c_target_C}
\end{equation*}
with $\ProbitU_f(\vx)$ defined as the previous the $C$ dimension unlearned model~\eqref{eq:dd_target_softmax2}.
In the end, we have a feature-agnostic unlearned model using only the empirical distributions.
Again, it makes little sense in practice, \ie\ for unseen testing data.
The KL distillation drives the unlearned NN to get closer to this target.
This results in a slight decrease of the probit related to $y(\vx)$ for the forget data.
This approximation makes experimentally a strong behavioral baseline for the subclass scenario.

\section{Our framework is a Generalization of existing methods}
\label{app:unify}

\subsection{Existing methods}
This appendix justifies the claim of Sect.~\ref{sec:related} that our framework~\eqref{eq:delta-M} generalises the
logit-difference methods recently introduced for unlearning.
The following cited papers are mainly empirical works and our framework gives them a theoretical justification and conditional statistical guarantees.
They use hyper parameters fixed heuristically or by grid search, whereas we provide a theoretically grounded procedure for setting the scaling of the unlearning signal (see Alg.~\ref{alg:etamax}).

The transposition of our framework to LLMs considers $\vx$ as
the context, $y$ as the next token, and $C$ as the vocabulary size.
For each paper, we identify our framework with the paper's notation.
Of note, two logits differing by a constant vector produce the same probabilities of next token using the softmax operator.

\paragraph{ULD~\citep{ji2024uld}.} It trains an expert assistant on the forget set. For the retain side it is equivalent to output the agnostic uniform law $\vu\doteq\frac{1}{C}\1$.
Its logit output is denoted $\ell_a$ and serves as a correction of the the initial model logit.
\begin{equation}
    \ell_f=\LogitI-\alpha \,\ell_a
\end{equation}
Adding the constant vector $\alpha\log \vu$, we identify $\Proxy=\softmax(\ell_a)$, $\ProxyR=\vu$, and $\eta=\alpha$.

\paragraph{UCD~\citep{suriyakumar2025ucd}.} It uses a corrupted model
$A_{\mathrm{corr}}$ on the forget set and into a safe model $A_{\mathrm{clean}}$ on the retain set.
\begin{equation}
    \log P_{\mathrm{aligned}}\propto\log P_{\mathrm{corr}}-\alpha(\log A_{\mathrm{corr}}-\log A_{\mathrm{clean}})
\end{equation}
We identify $\Proxy=A_{\mathrm{corr}}$, $\ProxyR=A_{\mathrm{clean}}$,
$\eta=\alpha$. Our Prop.~\ref{prop:tight-bound} is
that same statement, but generalized to imperfect proxies.

\paragraph{T3~\citep{block2026t3}.} Its transformation use a temperature $T\geq1$
\begin{equation}
    \hat p_r^{(T)}\propto\Ideal^{1/T}\hat f,
\end{equation}
and the classifier $\hat f$ estimates implicitely our unlearning signal.
This is our framework applied twice:
\begin{enumerate}[label=(\roman*)]
    \item \emph{Tempering}. We identify $\Proxy=\ProbitI$, $\ProxyR=\vu$, and $1-\eta=1/T$,
    since $$\log\ProbitI+\eta(\log \vu-\log\ProbitI)=\frac1T\log\ProbitI + \gamma\1$$
    \item \emph{Tilting}. Proxies are implicitly encoded into $\hat f(\vx,y)\simeq \Ideal\sach{r}{\vx,y}=\frac{\IdealR\sach{y}{\vx}\Ideal\sach{r}{\vx}}{\Ideal\sach{y}{\vx}}$ so that $$\log\hat f(\vx,y) = \left( \log\ProxyR\sach{y}{\vx}-\log\Proxy\sach{y}{\vx}\right) + \log\Proxy\sach{r}{\vx}$$
    with the implicit pair $(\ProxyR\sach{y}{\vx},\Proxy\sach{y}{\vx})=(\ProbitI\sach{y}{\vx}\hat f(\vx,y)/\Proxy\sach{r}{\vx},\ \ProbitI\sach{y}{\vx})$ used with $\eta=1$ and where
    \begin{equation}
    \Proxy\sach{r}{\vx}\doteq\sum_{c=1}^C\ProbitI\sach{c}{\vx}\hat f(\vx,c).
    \end{equation}
\end{enumerate}
Importantly, our procedure for $\eta_{\max}$ Alg.~\ref{alg:etamax} needs to compute the log ratio, renormalised as a log shift of two distributions
\begin{equation}
    \log\frac{\ProxyR\sach{y}{\vx}}{\Proxy\sach{y}{\vx}} = \log \hat f(\vx,y) - \log \sum_{c=1}^C \ProbitI\sach{c}{\vx}\hat f(\vx,c)
\end{equation}
and not both logs separately, so that it readily can be applied on the tilting of T3~\cite{block2026t3}.

Their choice of $\eta=1$ is justified only with ideal estimation of $\log\hat f$, and we infer that it is why they needed to introduce the additional tempering, instead of deriving an $\eta_{\max}$ potentially greater than $1$. Then, putting $\eta=\eta_{\max}$ instead of $\eta=1$ ensures $Z(\eta)\leq 0$ in case of approximated tilting $\hat f$.

\paragraph{Datamodel Matching (DMM)~\citep{georgiev2025attribute}.}
The proposed unlearning for classifiers proceeds in two steps.
The prediction step estimates the logits of the ideal retrained network thanks to a datamodel;
then the matching step fine-tunes $\vtheta$ towards that prediction.
The prediction step exactly instantiates~\eqref{eq:delta-M} with $\eta=1$ and with proxies taken from the datamodels.
The matching step corresponds to our optional distillation~\eqref{eq:lda2c-lstsq-grad}.
Alg.~\ref{alg:etamax} therefore applies unchanged.

Datamodel is the main tool of this work. A datamodel is a parametric function optimized to predict the output $\bm{\ell}(\vx; \mathcal{A}(S)) \in \R^C$ of a model trained on a subset $S\subseteq\train$ via learning algorithm $\mathcal{A}$ when inferring on a specific example $\vx$ ~\citep{ilyas2022datamodelspredictingpredictionstraining}. Linear datamodels read as:
\[
\hat{\bm\ell}_S(\vx) \doteq \bm\beta_0(\vx) + \sum_{i\in S}\bm\beta_i(\vx) \approx \bm{\ell}(\vx; \mathcal{A}(S)) \in \R^C.
\]
The computation of these coefficients for a given $\vx$ is fast but it requires a demanding offline preparation.
The unlearner samples $M$ (about thousands) random subsets $S_1, \dots, S_M \subseteq \train$ (e.g. random halves of $\train$), and trains a network on each $S_m$.
Online, for a given $\vx$, the unlearner solves the $\ell_1$-regularized regression problem for each class output:
\[
\left(\bm\beta_0(\vx), \bm\beta(\vx)\right) = \underset{(\bm w_0, \bm w) \in \R^{C} \times \R^{C \times |\train|}}{\arg\min} \frac{1}{M} \sum_{m=1}^M \left\| \bm w_0 + \bm w \mathbf{1}_{S_m} - \bm{\ell}(\vx; \mathcal{A}(S_m)) \right\|_2^2 + \lambda \|\bm w\|_1,
\]
where $\mathbf{1}_{S_m} \in \{0,1\}^{|\train|}$ is the indicator vector of subset $S_m$.

The unlearner uses the predictions for any $\vx\in\train_f$ as a correction of the initial model logits: $\vh(\vx) \doteq \LogitI(\vx) - \sum_{i\in\train_f}\bm\beta_i(\vx)$.

We show that DMM is indeed our target~\eqref{eq:delta-M} with the implicit scale $\eta=1$.
Consider two datamodels at $\train$ and
at $\train_r$.
Since they are linear and $\train_r=\train\setminus\train_f$, we have
$\hat{\bm\ell}_{\train_r}=\hat{\bm\ell}_{\train}-\sum_{i\in\train_f}\bm\beta_i$.
We then define our proxies as
$\Proxy\sach{\cdot}{\vx}\doteq\softmax(\hat{\bm\ell}_{\train}(\vx))$ and
$\ProxyR\sach{\cdot}{\vx}\doteq\softmax(\hat{\bm\ell}_{\train_r}(\vx))$.
In the end, we get back to our unlearned target $\LogitU_\eta(\vx) = \LogitI(\vx) + \eta (\log\Proxy\sachq{r}{\cdot}{\vx} - \log\Proxy\sach{\cdot}{\vx}) - \zeta(\vx;\eta)\1$ with the following identification:
\begin{eqnarray}
\label{eq:dmm-family}
\eta &=& 1,\\
\log\Proxy\sachq{r}{\cdot}{\vx} - \log\Proxy\sach{\cdot}{\vx}&=&-\sum_{i\in\train_f}\bm\beta_i(\vx),\\
\zeta(\vx;\eta) &=& \lse\left[\LogitI(\vx)-\eta\sum_{i\in\train_f}\bm\beta_i(\vx)\right].
\end{eqnarray}
In other words, $\ProbitU_1=\softmax(\vh)$.

\subsection{Application : Datamodel Matching as a use case}

\newcommand{\DmmTable}{%
\begin{tabular}{@{}l@{\hspace{3pt}}cc@{\hspace{3pt}}cc@{\hspace{3pt}}cc@{\hspace{3pt}}cc@{}}
\toprule
 & $\ProbitI$ & $\ProbitR$ & \multicolumn{2}{c}{DMM (original)} & \multicolumn{2}{c}{subclass DMM} & \multicolumn{2}{c}{subclass DMM -- inaccurate} \\
 & &  & $\eta=1$ & $\eta_{\max}$ (ours) & $\eta=1$ & $\eta_{\max}$ (ours) & $\eta=1$ & $\eta_{\max}$ (ours) \\
\midrule
\multicolumn{9}{@{}l}{\textit{DINOv2 linear} (5 seeds)} \\
\quad $\eta$ & -- & -- & 1 & $0.71{\scriptstyle\pm 0.02}$ & 1 & $0.94{\scriptstyle\pm 0.02}$ & 1 & $2.38{\scriptstyle\pm 0.04}$ \\
\quad $\mathrm{Acc}^{f}$ & $99.41{\scriptstyle\pm 0.01}$ & $51.55{\scriptstyle\pm 0.51}$ & $98.76{\scriptstyle\pm 0.01}$ & $98.85{\scriptstyle\pm 0.02}$ & \cellcolor{gray!25}$\boldsymbol{48.59{\scriptstyle\pm 1.90}}$ & $55.91{\scriptstyle\pm 0.16}$ & $94.31{\scriptstyle\pm 0.09}$ & \cellcolor{gray!25}$\boldsymbol{54.73{\scriptstyle\pm 0.11}}$ \\
\quad $\mathrm{Acc}^{t}$ & $99.62{\scriptstyle\pm 0.02}$ & $94.91{\scriptstyle\pm 0.15}$ & $99.57{\scriptstyle\pm 0.02}$ & $99.56{\scriptstyle\pm 0.02}$ & \cellcolor{gray!25}$\boldsymbol{94.84{\scriptstyle\pm 0.22}}$ & \cellcolor{gray!25}$95.43{\scriptstyle\pm 0.07}$ & $99.39{\scriptstyle\pm 0.02}$ & \cellcolor{gray!25}$\boldsymbol{95.43{\scriptstyle\pm 0.04}}$ \\
\quad $\KL^{f}$ & $5.46{\scriptstyle\pm 0.16}$ & -- & $5.44{\scriptstyle\pm 0.16}$ & $5.44{\scriptstyle\pm 0.16}$ & \cellcolor{gray!25}$\boldsymbol{0.13{\scriptstyle\pm 0.00}}$ & \cellcolor{gray!25}$0.14{\scriptstyle\pm 0.01}$ & $2.69{\scriptstyle\pm 0.13}$ & $\boldsymbol{0.33{\scriptstyle\pm 0.02}}$ \\
\quad $\KL^{t}$ & $0.52{\scriptstyle\pm 0.02}$ & -- & $0.52{\scriptstyle\pm 0.02}$ & $0.52{\scriptstyle\pm 0.02}$ & \cellcolor{gray!25}$\boldsymbol{0.01{\scriptstyle\pm 0.00}}$ & \cellcolor{gray!25}$0.01{\scriptstyle\pm 0.00}$ & $0.26{\scriptstyle\pm 0.01}$ & $\boldsymbol{0.04{\scriptstyle\pm 0.00}}$ \\
\midrule
\multicolumn{9}{@{}l}{\textit{DINOv2 MLP1} (5 seeds)} \\
\quad $\eta$ & -- & -- & 1 & $0.95{\scriptstyle\pm 0.00}$ & 1 & $1.03{\scriptstyle\pm 0.01}$ & 1 & $1.59{\scriptstyle\pm 0.02}$ \\
\quad $\mathrm{Acc}^{f}$ & $100.00{\scriptstyle\pm 0.00}$ & $59.08{\scriptstyle\pm 1.36}$ & $98.84{\scriptstyle\pm 0.02}$ & $98.92{\scriptstyle\pm 0.01}$ & \cellcolor{gray!25}$\boldsymbol{46.69{\scriptstyle\pm 1.53}}$ & \cellcolor{gray!25}$43.36{\scriptstyle\pm 0.41}$ & $88.11{\scriptstyle\pm 0.48}$ & $\boldsymbol{41.61{\scriptstyle\pm 0.24}}$ \\
\quad $\mathrm{Acc}^{t}$ & $99.81{\scriptstyle\pm 0.05}$ & $95.85{\scriptstyle\pm 0.12}$ & $99.80{\scriptstyle\pm 0.00}$ & $99.80{\scriptstyle\pm 0.00}$ & \cellcolor{gray!25}$\boldsymbol{94.51{\scriptstyle\pm 0.26}}$ & $94.04{\scriptstyle\pm 0.07}$ & $98.79{\scriptstyle\pm 0.10}$ & \cellcolor{gray!25}$\boldsymbol{94.19{\scriptstyle\pm 0.04}}$ \\
\quad $\KL^{f}$ & $5.03{\scriptstyle\pm 0.20}$ & -- & $4.80{\scriptstyle\pm 0.20}$ & $4.81{\scriptstyle\pm 0.20}$ & \cellcolor{gray!25}$\boldsymbol{0.20{\scriptstyle\pm 0.03}}$ & \cellcolor{gray!25}$0.26{\scriptstyle\pm 0.02}$ & $0.84{\scriptstyle\pm 0.07}$ & $\boldsymbol{0.35{\scriptstyle\pm 0.03}}$ \\
\quad $\KL^{t}$ & $0.49{\scriptstyle\pm 0.02}$ & -- & $0.48{\scriptstyle\pm 0.02}$ & $0.48{\scriptstyle\pm 0.02}$ & \cellcolor{gray!25}$\boldsymbol{0.02{\scriptstyle\pm 0.00}}$ & \cellcolor{gray!25}$0.03{\scriptstyle\pm 0.00}$ & $0.09{\scriptstyle\pm 0.01}$ & $\boldsymbol{0.04{\scriptstyle\pm 0.00}}$ \\
\midrule
\multicolumn{9}{@{}l}{\textit{ResNet-18} (5 seeds)} \\
\quad $\eta$ & -- & -- & 1 & $0.50{\scriptstyle\pm 0.35}$ & 1 & $0.80{\scriptstyle\pm 0.26}$ & 1 & $1.47{\scriptstyle\pm 0.34}$ \\
\quad $\mathrm{Acc}^{f}$ & $96.04{\scriptstyle\pm 5.16}$ & $45.56{\scriptstyle\pm 15.55}$ & $90.39{\scriptstyle\pm 5.53}$ & $93.78{\scriptstyle\pm 4.08}$ & $39.30{\scriptstyle\pm 3.80}$ & \cellcolor{gray!25}$\boldsymbol{50.55{\scriptstyle\pm 10.80}}$ & $65.92{\scriptstyle\pm 6.90}$ & \cellcolor{gray!25}$\boldsymbol{43.93{\scriptstyle\pm 8.10}}$ \\
\quad $\mathrm{Acc}^{t}$ & $95.90{\scriptstyle\pm 1.36}$ & $92.45{\scriptstyle\pm 1.39}$ & $95.91{\scriptstyle\pm 1.41}$ & $95.91{\scriptstyle\pm 1.36}$ & $91.52{\scriptstyle\pm 1.66}$ & \cellcolor{gray!25}$\boldsymbol{92.64{\scriptstyle\pm 0.36}}$ & $93.59{\scriptstyle\pm 1.42}$ & \cellcolor{gray!25}$\boldsymbol{91.92{\scriptstyle\pm 0.42}}$ \\
\quad $\KL^{f}$ & $3.09{\scriptstyle\pm 1.08}$ & -- & $2.61{\scriptstyle\pm 0.93}$ & $2.79{\scriptstyle\pm 0.91}$ & $0.62{\scriptstyle\pm 0.43}$ & \cellcolor{gray!25}$\boldsymbol{0.54{\scriptstyle\pm 0.22}}$ & $0.84{\scriptstyle\pm 0.18}$ & \cellcolor{gray!25}$\boldsymbol{0.58{\scriptstyle\pm 0.19}}$ \\
\quad $\KL^{t}$ & $0.39{\scriptstyle\pm 0.08}$ & -- & $0.39{\scriptstyle\pm 0.08}$ & $0.39{\scriptstyle\pm 0.08}$ & $0.20{\scriptstyle\pm 0.16}$ & \cellcolor{gray!25}$\boldsymbol{0.17{\scriptstyle\pm 0.10}}$ & $0.21{\scriptstyle\pm 0.11}$ & \cellcolor{gray!25}$\boldsymbol{0.19{\scriptstyle\pm 0.11}}$ \\
\bottomrule
\end{tabular}
}

\paragraph{DMM for semantic scenarios.}
DMM randomly halves the training set for the creation of the datamodels.
This is fine for unlearning one or some training data randomly picked, but it is not appropriate when unlearning training data from a subclass or a class.
In these scenarios, the forget data are correlated because they share the same semantic meaning.
DMM was not designed for such "semantic" scenarios.
To be fair, we propose a "subclass DMM" consisting in adapting the sampling of the subsets $S_1, \dots, S_M \subseteq \train$.
We split training data into subclass groups which are included \emph{as a whole} in the subset $S_m$ with probability $\nicefrac{1}{2}$.

\paragraph{Protocol.}
We compare the original DMM with $M=1000$, our "subclass DMM" version  with $M=200$ classifiers trained until convergence ($99\%$ training accuracy), and the same version but with an early stop of the training ($80\%$ training accuracy).
The latter setup artificially creates inaccurate proxies. The scaling $\eta$ is either set to $1$ (original DMM), to $\eta_{\max}$, or by a grid search minimizing the respective $\KL$ (impossible in practice as it needs one ideal retrained model).
These methods are applied to DINOv2 (linear and mlp1) and ResNet-18 based architecture, with the objective of unlearning the subclass "airplanes" from CIFAR-10.
Results are averaged over 5 seeds.

\paragraph{Results.}
Table~\ref{tab:dmm} shows that the original DMM barely unlearns even when using $M=1\,000$ models.
When fixing this problem with our accurate "subclass DMM" version, it turns out that setting $\eta=1$ is a good choice when the proxies are accurate with DINOv2. While for ResNet-18, the adaptive $\eta_{\max}$ here prevents over unlearning.
Otherwise in the inaccurate setup, our adaptive scaling takes a value different than $1$ and brings better performances. This happens when the proxies are inaccurate because derived from poor datamodels ("subclass DMM - inaccurate" column) or because modelling less informative features (ResNet-18 architecture).

Figure~\ref{fig:dmm} shows that our adaptive scaling offers performances closer to the optimal ones, especially when dealing with inaccurate proxies.
The optimal performance level $\eta$ is measured by cheating, \ie\ by resorting to an ideal retrained model and by minimizing $\KL^\FORGET$ or $\KL^\TRAIN$.

\begin{table}[bt]
\caption{Comparison of DMM versions (original, "subclass" with accurate or inaccurate datamodels) with or without our adaptive scaling, on \subclass{} airplanes unlearning (CIFAR-10). Mean $\pm$ std over the 5 seeds.
Bold (over a column), Grey (best two over a row): smaller $\KL$, accuracy nearer to retrained $\ProbitR$.}
  \centering\scriptsize\setlength{\tabcolsep}{2pt}
  \resizebox{\linewidth}{!}{\DmmTable}
  \label{tab:dmm}
\end{table}

\begin{figure}[bt]
  \centering
  \includegraphics[width=\columnwidth]{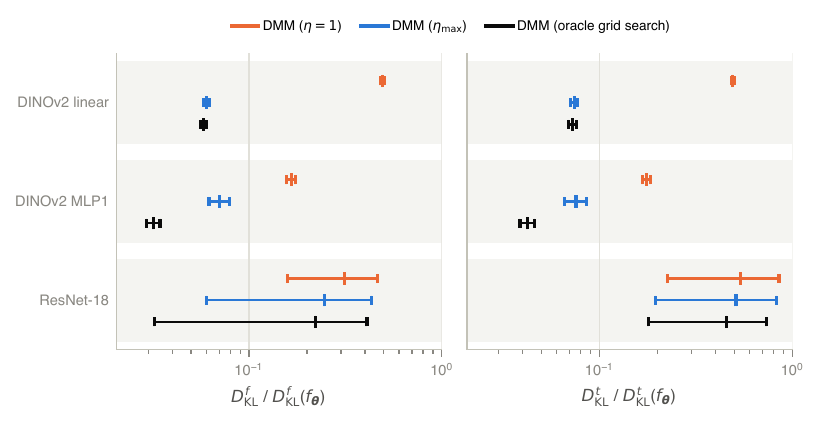}
  \caption{Comparison of "subclass DMM" with or without adaptive scaling, on \subclass{} airplanes unlearning (CIFAR-10).
  Tick and whisker: mean $\pm$ std over the seeds. The oracle grid search needs one ideal retrained model $\ProbitR$.}
  \label{fig:dmm}
\end{figure}

\section{Handling large NN}
\label{app:random_projection}

This appendix shows that scaling our framework to networks whose hidden
representations are too high-dimensional for a direct $\lda/\qda$ fit
requires no theoretical change. It is only a matter of choice of the proxies.

\paragraph{Composed proxy}

The whole machinery of Sect.~\ref{sec:our_method} only requires modeling the input distribution so that
the proxy shift can be computed in closed form.
Up to now we instantiated this proxy directly on $\vx$ via
$\lda$ or $\qda$ for instance. When $\vx$ lives in a space too large for a class-conditional
Gaussian fit, we simply compose:
\begin{equation}
    \Proxy\sach{\vx}{y}
    \;\doteq\;
    \Proxy\sachq{0}{F(\vx)}{y},
    \qquad
    F\colon \sX \to \R^k,
    \label{eq:rp_proxy}
\end{equation}
where $\Proxy_{0}$ is the standard $\lda$ or $\qda$ proxy on the space $F(\sX)$ and $F$ is any fixed injection from the high-dimensional input space onto a low-dimensional target.
The retain proxy $\ProxyR$ is built identically with the same $F$.
Plugging~\eqref{eq:rp_proxy} into the framework yields a logit shift
that is computed in $\R^k$.

The assumptions and the bounds of Sect.~\ref{sec:our_method} carry over verbatim.
Any $F$ that yields a low-dimensional representation on which $\lda/\qda$ is accurate is admissible.
Indeed, our framework never constrains the form of
the proxy: it only asks for a pair $(\Proxy,\ProxyR)$ of models of the
class-conditional law whose posterior log-ratio is available in
closed form. For a fixed, data-independent $F$, the composition
$\Proxy\sach{\vx}{y}\doteq\Proxy\sachq{0}{F(\vx)}{y}$ is such a model, hence a
perfectly legitimate proxy.

Projections help us compete with the iterative methods not only in terms of performance, but also in runtime ans shown in Fig.~\ref{fig:cost}.

\begin{figure}[tb]
  \centering
  \includegraphics[width=0.8\columnwidth]{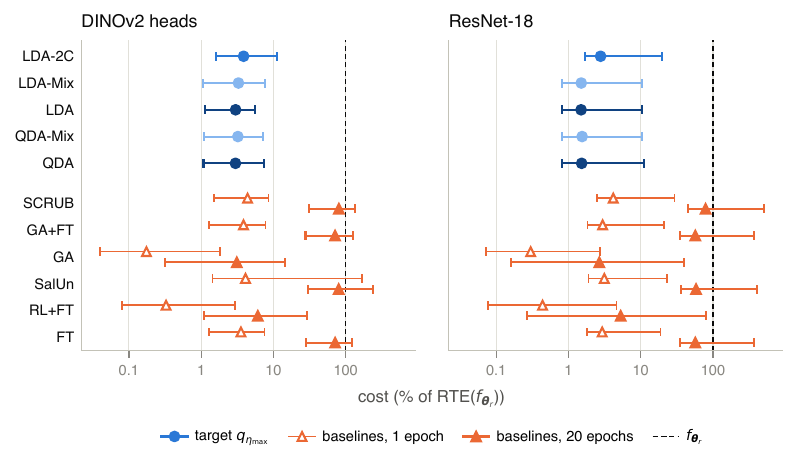}
  \caption{Comparison of the runtime of deriving our black-box unlearned model to the runtime of the baselines. For ResNet-18 : $k=512$ on layer $\ell = 3$}
  \label{fig:cost}
\end{figure}

\paragraph{Practical instance for ResNet-18.}

The specific $F$ we use in the ResNet-18 experiments (Sect.~\ref{sec:experiments}) is anecdotal, it is
one convenient choice among many. We take
\begin{equation}
    F(\vx) \;=\; \mOmega\,\phi^{(3)}(\vx),
    \label{eq:rp_F}
\end{equation}
where $\phi^{(3)}$ is the activation just after the third residual stage
($\dim\sim 8{\times}10^3$ at our resolution) and
$\mOmega\!\in\!\R^{k\times d}$ ($k\!=\!512$) is a deterministic
semi-orthogonal Johnson--Lindenstrauss
sketch~\citep{johnson1984extensions,dasgupta2003elementary} obtained as the
orthogonal factor of a Gaussian matrix via QR decomposition, with a fixed seed shared across all proxy variants.

\paragraph{Where to tap the proxies.} The hook depth ($\phi^{(3)}$) is the standard
mid-level/pre-classification choice: earlier stages have too small a
receptive field, producing very long features which are far from a Gaussian distribution.
Later ones already encode logits and would make the proxies redundant with $\LogitI$ and lack the information of membership. Fig.~\ref{fig:layerablation} show that the resulting target is very dependant on the layer quality. Near the output layer, the membership information is more compress and the unlearning signal is indistinguishable.
\begin{figure}[b]
  \centering
  \includegraphics[width=0.8\columnwidth]{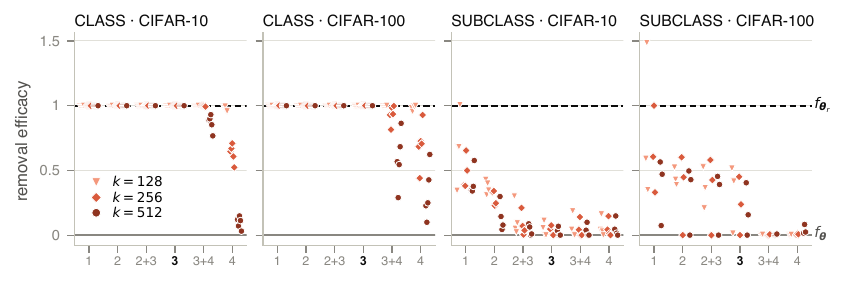}
  \caption{Removal efficacy, defined as $\big(\Acc^{\FORGET}(\ProbitI) - \Acc^{\FORGET}(\ProbitU)\big) /
\big(\Acc^{\FORGET}(\ProbitI) - \Acc^{\FORGET}(\ProbitR)\big)$, vs. layer $\ell$ (or combination) for the proxy \LDAIJ.
Projection dimension $k \in\{ 128, 256, 512\}$. \textbf{Bold}: Sect.~\ref{sec:experiments} uses layer $\ell = 3$.}
  \label{fig:layerablation}
\end{figure}
Another phenomenon is that the best layer depends on the task, and its granularity.
Figure~\ref{fig:layerablation} highlight that the behavior is very different between the \class\ and \subclass\ scenarios.
Along the network, the membership information is compressed until it vanishes.
\subclass\ scenario needs earlier layers to perform well because the information to extract is thinner: \subclass\ of a super-class.

Section~\ref{sec:experiments} only uses the layer $\ell = 3$ and projection dimension $k=512$.
This choice is deliberately shown to be sub-optimal in the \subclass\ scenario (see. Fig.~\ref{fig:layerablation}).
The framework would benefit from reducing the dimension $k$ and selecting an earlier layer.
This illustrates how the problematic of unlearning is shifting towards features representation/selection.

\paragraph{Choice of features dimensions.}
Intuitively, one has to match the feature dimensions with the quantity of data available for identifying the model.

For Gaussian proxies, a $k\times k$ class
covariance needs at least  $k$ samples to be non-singular so the per-\class\ data budget
$|\train(y)|$ caps $k$.
On the CIFAR-10 \subclass\ task, whose two
superclasses hold $25,000$ images each, $k=512$ is comfortable.
On CIFAR-100, there are  $500$ images per class, \ie\ less than $k$.
A proxy sharing the covariance matrices is preferred.
Indeed, the limiting factor is the size of the forget set, whose per-\class\ counts
$|\train_f(y)|$ might be very small.
This is one more reason to prefer the single
shared covariance of \LDAIJ\ or \LDAI\ to the per-state or per-\class\ covariances of
\LDAII/\QDAII.

\paragraph{Conclusion.}
This framework scales to arbitrary backbones for a fixed cost in $O(k)$ because the proxy fit is performed in $\R^k$ regardless of the input
dimension. The features is what matters most in practice when modeling the data in order to create proxies.
Operator $F$ is the only choice that needs revisiting for a new
architecture: any low-dimensional representation on which $\lda/\qda$ is reasonable.
The tap layer is more important than the dimension of the projection (see Fig.~\ref{fig:layerablation}).

\section{Privacy auditing}
Our unlearning procedure is designed to bring guarantees on the behavior of the unlearned model.
It is not a priori privacy oriented. Yet, the following section discusses the connection between privacy auditing and behavioral guarantees.
The other sections experimentally suggest that our unlearning procedure is beneficial from the privacy point of view.

\subsection{Membership Inference Attacks and KL divergence}
\label{app:stein}
This section relates the KL divergence, the core metric in this paper, to privacy auditing.

Imagine that an attacker has access to a classifier in a black box.
He knows that this black box contains either the ideal model retrained from scratch $\ProbitR$ or the unlearned NN $\ProbitD$.
He queries the black box with some random data $\vx$ and observes the output $\hat{y}$, the predicted label.
Of note, this is a worst-case scenario for the attacker. MIAs in the literature agree on a more favorable scenario where the attack has access to the probit or logit vector.

This section relates $\KL^f$ (or $\KL^t$), as reported in the experimental results,
to the number of queries necessary to identify the classifier in the black-box with a given reliability.

For a given input $\vx$, we have the following hypothesis:
\begin{equation*}
        \gH_0 : \hat{y}\sim \ProbitR\sach{y}{\vx} \quad \text{versus}\quad
    \gH_1: \hat{y}\sim \ProbitD\sach{y}{\vx}
\end{equation*}

Let us denote the probability of a false positive by $\alpha$ and the probability of a false negative by $\beta$. The attacker's test outputs a binary decision $\rd\in\{0,1\}$, which has the following Bernoulli distribution:
\begin{description}
    \item[$\gH_0$:] The test makes a \emph{false positive} whenever $\rd=1$. Then, $\rd\sim\bernoulli(\alpha)$.
    \item[$\gH_1$:] The test makes a \emph{false negative} whenever $\rd=0$. Then, $\rd\sim\bernoulli(1-\beta)$.
\end{description}
Suppose that the attacker designs the test to achieve an Equal Error Rate, \ie\ $\alpha=\beta$, since the error rates $\alpha$ and $\beta$ matter equally.

The data processing theorem states that any processing of the data $\hat{y}$ reduces the KL divergence.
The attacker's test is a processing; therefore:
\begin{align*}
    \kl[y\mid\vx]{\IdealR}{\ProbitD} &\geq \kl{\bernoulli(\alpha)}{\bernoulli(1-\beta)}=(1-\alpha)\log\left(\frac{1-\alpha}{\beta}\right)+\alpha\log\left(\frac{\alpha}{1-\beta}\right)\\
    &\stackrel{(\alpha=\beta)}{=}(1-2\alpha)\log\left(\frac{1-\alpha}{\alpha}\right).
\end{align*}
This inequality is at the root of the Stein lemma.

Querying $N>1$ inputs brings more information. Provided they are independent, their KL divergences sum up.
This sum can be written as $N\cdot\E\kl[y\mid\vx]{\ProbitR}{\ProbitD}$, where the expectation is indeed an average over the queries as reported experimentally in Sect.~\ref{sec:experiments}.
At the end, we get a necessary number of queries to reach a reliability given by $1-\alpha$:
\begin{equation*}
    N \geq \frac{(1-2\alpha)\log\left(\nicefrac{1-\alpha}{\alpha}\right)}{\E\kl[y\mid\vx]{\ProbitR}{\ProbitD}}.
\end{equation*}
This number comes with two flavors. The attacker may query test data (resp. forget data) and we plug $\KL^t$ (resp. $\KL^f$) in the formula.
For instance, for $\alpha = 1/1\,000$, $N_t = 138$ for $\KL^t=0.05$, $N_f=69$ for $\KL^f=0.1$.
The key question is whether $|\train_f|<N_f$; if so, the attacker cannot distinguish the two classifiers with forget data, even if he knows all of them.

\paragraph{Conclusion.}
The basic rationale of this study is the following: If an attacker is not able to distinguish $\ProbitD$ from a model $\ProbitR$,
which never processed $\train_f$, then he can not infer anything about the forget set.
This rationale is not true because this analysis is not a privacy audit.
It delivers no per-example statement and no
$(\varepsilon,\delta)$ statement either.

First, the quantity $N_f$ is based on an \emph{average} over the forget set.
This is a so-called \textit{population} MIA in~\citep{Hayes2024ulira} or \textit{average-case} evaluation in~\citep{naderloui2025ruli}.
A handful of forget samples may remain perfectly identifiable, typically the ones
per-example audits of the U-LiRA family are able to expose.

On the second point, our unlearning is indeed not compliant with Differential Privacy.
If the attacker knows $\train$, $\ProbitI$, $\ProbitD$, and the unlearning procedure, he can theoretically identify $\train_f$ by an exhaustive search because our unlearning is mainly deterministic.
Randomness happens only in the last optional distillation step.

\subsection{Per-sample audit}
\label{sec:peraudit}

We evaluate a per-sample membership inference attack
which decide that the sample $\vx\in\train_f$ is not a training data if
$\kl[y\mid\vx]{\ProbitR}{\ProbitD} < \kl[y\mid\vx]{\ProbitI}{\ProbitD}$. It is our metric "Audit" in table Tab.~\ref{tab:hero-benchmark}. Audit measures the percentage of forget data that verifies this inequality.
As observed in Sect.~\ref{sec:exp_benchmark}, Fig.~\ref{fig:ribero} confirms that the \subclass\ scenario is more difficult.
The difficulty transfers from reducing the KL on average, to reducing it per sample.
It is also tangible when $\ProbitD = \ProbitR^\prime$, a model trained from scratch independently of $\ProbitR$ (first line of Fig.~\ref{fig:ribero}). It is measuring something different than our Prop.~\ref{prop:decrease} or our other KL metrics. For instance GA performs well overall on this audit while destroying the model.

\begin{figure}[bt]
  \centering
  \includegraphics[width=\columnwidth]{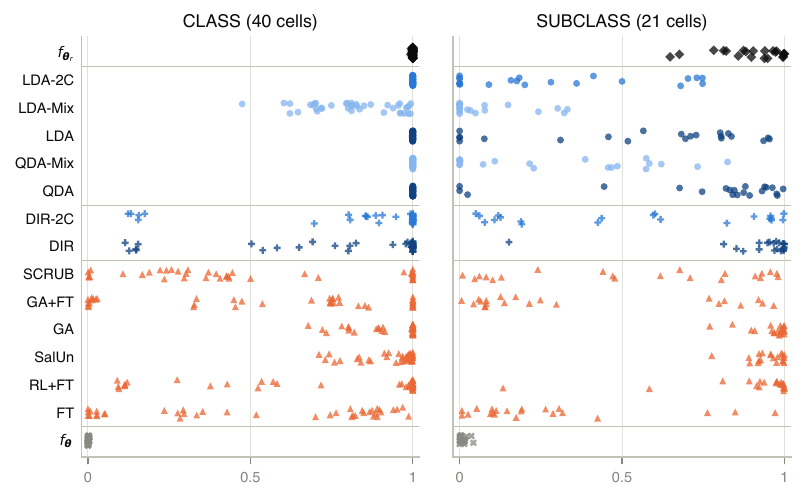}
  \caption{$X$-axis is the fraction of samples in $\train_f$ that are not flagged as training data. One mark per run across scenarios and architectures.}
  \label{fig:ribero}
\end{figure}

On this audit, the naive proxies seem to be stronger than the 2C construction.
Our explanation is as follows: the identification of $\ProxyR$ from naive constructions never sees data from the forget set.
On the contrary, the 2C construction by design create a single fit where the identification of $\ProxyR$ sees forget data.
The audit suggests that, for privacy purposes, the $\ProxyR$ may not be contaminated with forget data.

\paragraph{Scope of this audit.} This per-sample test is weaker than a U-LiRA audit~\cite{Hayes2024ulira} and should not be read as a privacy guarantee. As discussed in the previous section, our target is deterministic given the proxies, so an attacker knowing $\train$, $\ProbitI$ and the procedure could in principle identify $\train_f$.

\section{Machine unlearning overview}
\label{app:machine_unlearning}

This appendix expands the related work presented in Sect.~\ref{sec:related}.

\subsection{Machine unlearning}
The literature is now broad enough that several survey papers
cover the field from complementary angles~\citep{Nguyen2022survey,Xu2024surveyTaxonomy,Liu2024federatedSurvey,Liu2025rethinkingLLM}.

\paragraph{Motivation and regulatory context.}
The renewed interest in machine unlearning is driven by data-protection
regulations such as the EU General Data Protection Regulation (GDPR) and its
"right to be forgotten" \citep{Voigt2017gdpr},
together with the California Consumer Privacy Act (CCPA) and similar
frameworks worldwide. Beyond legal compliance, unlearning has been
also identified as a tool for removing the influence of poisoned, biased,
mislabelled, copyrighted or otherwise undesirable training samples
\citep{cao2015machine,Bourtoule2021SISA,Eldan2023whoIsHarryPotter,Goel2024corrective},
and as a primitive for AI safety (e.g., suppressing knowledge in
foundation models~\citep{Li2024wmdp}).

\paragraph{Exact unlearning.}
Machine unlearning was formalised by \citet{cao2015machine}, who introduced
the term and proposed summation-form algorithms that allow data removal
without full retraining. The gold standard remains exact removal, in
which the unlearned model is statistically indistinguishable from a model
retrained from scratch on the retained set
\citep{cao2015machine,Bourtoule2021SISA,Ginart2019makingAIforget}.
\citet{Bourtoule2021SISA} achieve
exact unlearning in deep models via SISA training, which partitions the training data
into shards so that only the affected shard needs to be retrained on
deletion. Although SISA reduces retraining cost by isolating deletions, it
requires storing intermediate checkpoints for every shard, and at scale this
storage overhead and the accuracy loss caused by aggressive sharding make it
impractical for large deep networks.

\paragraph{Gradient-based approximate unlearning.}
To avoid retraining, the dominant paradigm relies on gradient-based weight
updates.
\citet{Thudi2022unrolling} provide an unrolled-SGD analysis of the verification
error and show how a training-time penalty can ease subsequent unlearning.
Knowledge-distillation variants jointly enforce forgetting and retention:
SCRUB \citep{Kurmanji2023scrub} trains a student that minimises the KL
divergence to the original model on the retained data while maximising the behavioral change on the forget set. \citet{Chundawat2023badteacher} use teachers whose soft labels guide the student to suppress class
knowledge. SalUn \citep{Fan2024salun} restricts updates to the weights
identified as most responsible for the forgotten data via gradient-based
saliency, reaching state-of-the-art results on image classification and
generation. \citet{Jia2023sparseUnlearn} show that model sparsification
prior to unlearning closes the gap between exact and approximate unlearning
across multiple algorithms. \citet{Tarun2023fast} propose UNSIR, an
impair-repair scheme in which an error-maximising noise first degrades
class knowledge and a repair step then restores retain-set accuracy.
\citet{Neel2021descentToDelete} provide gradient-descent-based deletion
algorithms that handle adversarial sequences of deletion requests for convex
models. Beyond their empirical strengths, these first-order methods share a structural
limitation: the update follows an optimisation trajectory that admits no
closed-form characterisation in terms of the data distribution or the
resulting predictive shift.

\paragraph{Fisher information and synaptic dampening.}
Fisher-based methods identify which parameters encode the forgotten data and
selectively modify them. \citet{golatkar2020eternal} derive a Fisher
Information Matrix (FIM) scrubbing procedure that perturbs weights in order
to remove cohort information, and \citet{golatkar2020forgettingOutsideBox}
extend this analysis via a Neural Tangent Kernel linearisation that better
handles the null-space of the weights and bounds the information accessible
from input-output observations. \citet{Foster2024ssd} propose Selective
Synaptic Dampening (SSD), a retrain-free two-step method that uses the FIMs
of the training and forget sets to select and dampen the parameters most
important to the forget set. SSD is computationally light and does not
require retained data at unlearning time, but its forgetting is
coupled to the quality of the diagonal FIM approximation, which does not
expose the structure of the data distribution.
\citet{Mehta2022deepUnlearningHessian} introduce L-CODEC, a randomised
conditional-independence test that selects a subset of the model's parameters,
making Newton-style updates tractable for very large vision and NLP models.

\paragraph{Influence functions and certified removal.}
Influence functions \citep{Koh2017influence}, borrowed from robust
statistics, approximate the leave-one-out parameter shift via a single
Newton step on the training loss. \citet{guo2020certified} exploit this
connection to derive certified data removal for $\ell_2$-regularised linear
models using a calibrated noise mechanism masks the residual Newton-step error.
\citet{Sekhari2021remember} extend the theoretical analysis by bounding
generalisation after unlearning, and \citet{Izzo2021approximateDeletion}
propose a projection-based deletion algorithm whose cost is linear in
the feature dimension and independent of the dataset size for linear and
logistic models. \citet{Warnecke2023featuresLabels} extend influence-based
unlearning to the deletion of features and labels, providing closed-form
updates and certified guarantees for strongly convex losses.
Beyond linear models, extending influence-function guarantees to non-convex
deep networks remains challenging \citep{Koh2017influence,guo2020certified}:
the Hessian approximation is hard to control, and no closed-form bound on
the predictive shift in terms of the data distribution is available. More
recent work such as \citet{Chien2024certified} uses Langevin dynamics to
obtain certified unlearning without convexity assumptions, but the resulting
bounds still ignore the class-conditional structure of the data.

\paragraph{Geometry-based and training-free unlearning.}
A recent direction exploits the geometric structure of trained
representations to obtain closed-form, training-free corrections.
\citet{Le2025pour} build on Neural Collapse theory \citep{Papyan2020collapse}
to show that the orthogonal projection of a simplex Equiangular Tight Frame
(ETF) onto the subspace of retained classes remains an ETF in the
lower-dimensional space, which yields a provably optimal forgetting
operator. Their method, POUR-P, delivers this correction in closed form and
without gradient steps. In a similar spirit,
\citet{kodge2024deepunlearningfastefficient} project model weights onto the
subspace orthogonal to the principal directions of the forget class,
rendering the model blind to that information without retraining. While
these methods are fast and elegant, the optimality guarantees of POUR rely on the terminal
neural-collapse regime, and neither family provides a data-distribution
bound on the residual unlearning error.

\paragraph{Theoretical guarantees and certification criteria.}
\citet{guo2020certified, Sekhari2021remember} measure unlearning success as
$(\eps,\delta)$-indistinguishability, mirroring differential privacy
\citep{Dwork2014dp}. KL divergence between the ideal
retrained model and the empirically unlearned model has since been adopted
as natural criteria
\citep{golatkar2020eternal}, but
nearly all theoretical results are stated in terms of algorithmic
properties (Hessian approximation quality, loss Lipschitz constants) rather
than in terms of the data distribution. A separate line of work studies the
auditability of unlearning definitions and the difficulty of certifying
forgetting in arbitrary training pipelines
\citep{Thudi2022auditableDefinitions},
underlining the gap between algorithmic and distributional guarantees.

\paragraph{Membership inference as an evaluation tool.}
Membership inference attacks (MIAs) were introduced by
\citet{Shokri2017mia}, formally connected to overfitting by
\citet{Yeom2018overfittingMIA}, and sharpened by the likelihood-ratio attack
LiRA of \citet{Carlini2022lira}, which achieves the highest statistical
power at low false-positive rates thanks to shadow models. MIAs
have become the standard black-box audit for unlearning efficacy
\citep{Triantafillou2024challenge,Hayes2024ulira}: an effective unlearning
should prevent an adversary from distinguishing forgotten examples from
held-out non-members. \citet{Hayes2024ulira} show that population-level MIAs
systematically overestimate the privacy protection of approximate methods,
and they propose U-LiRA, a per-example extension that serves as a stricter audit. Complementary evaluation protocols have also been
developed: \citet{Goel2022adversarialEval} introduce the Interclass
Confusion test together with the EU-$k$ and CF-$k$ baselines, and the
findings of the first NeurIPS unlearning competition
\citep{Triantafillou2024challenge} provide a broad empirical assessment of
the state of the art under a unified evaluation pipeline.

\paragraph{Unlearning beyond classification: generative models and LLMs.}
Although the present work focuses on classification, machine unlearning has
recently been extended to generative models, where it is used to suppress
copyrighted, sensitive or harmful content. For diffusion models,
\citet{Gandikota2023erasingConcepts} propose Erased Stable Diffusion (ESD),
a fine-tuning method that removes a target concept from the model weights
using only a textual description and negative guidance. SalUn
\citep{Fan2024salun} also covers concept erasure in diffusion models. For
large language models, \citet{Eldan2023whoIsHarryPotter} introduce the
"Who's Harry Potter?" protocol for unlearning copyrighted content.
\citet{Maini2024tofu} propose the TOFU benchmark of fictitious authors for
controlled evaluation, and \citet{Li2024wmdp} introduce the WMDP benchmark
for hazardous-knowledge unlearning together with the RMU representation
misdirection method.

\end{document}